\documentclass[lettersize,journal]{IEEEtran}
\usepackage{amsmath}
\usepackage{algorithmic}
\usepackage{algorithm}
\usepackage{array}
\usepackage{stfloats}
\usepackage{url}
\usepackage{multirow}
\usepackage{graphicx}
\usepackage{amsthm}
\fnbelowfloat
\usepackage{booktabs}
\usepackage[table]{xcolor}
\usepackage[noadjust]{cite}
\usepackage{amssymb}
\usepackage[colorlinks,linkcolor=red,anchorcolor=blue,citecolor=green]{hyperref}
\newtheorem{definition}{Definition}
\newtheorem{remark}{Remark}

\newtheorem{theorem}{\textbf{Theorem}}
\newtheorem{lemma}{\textbf{Lemma}}

\newcommand{\mat}[1]{\ensuremath{\mathbf{#1}}}  
\newcolumntype{C}{@{\extracolsep{\fill}}c}
\DeclareMathOperator{\fold}{fold}  

\begin{document}

\title{Pre-Trained Low-Rank Tensor Decomposition for Multi-Dimensional Image Recovery
	\thanks{This research is supported by National Key R\&D Program of China (2025YFA1016400), NSFC (No. 12371456), Sichuan Science and Technology Program (No.  2024NSFSC0038).
	\emph{(Corresponding authors: Ting-Zhu Huang and Xi-Le Zhao.)}
	}
	\thanks{Bing-Zhang Fu, Zhi-Long Han, Ting-Zhu Huang, and Xi-Le Zhao are with the School of Mathematical Sciences, University of Electronic Science and Technology of China, Chengdu 611731, China (e-mail: bingzhangfu619@163.com; hzlmath@163.com; tingzhuhuang@126.com; xlzhao122003@163.com).}
	\thanks{Deyu Meng is with the School of Mathematics and Statistics and Ministry	of Education Key Lab of Intelligent Networks and Network Security, Xi’an   Jiaotong University, Xi’an, Shaanxi 710049, China, and also with Pazhou	Laboratory (Huangpu), Guangzhou, Guangdong 510335, China (e-mail: dymeng@mail.xjtu.edu.cn).}
}
\author{Bing-Zhang Fu, 
	Zhi-Long Han, 
	Ting-Zhu Huang,
	Xi-Le Zhao,
	and Deyu Meng
}
\maketitle

\begin{abstract}
Recently, tensor decompositions are prevalent for multi-dimensional image representation, which learn the instance-specific structure of each image from scratch.
However, tensor decompositions neglect the common structure across different images, leading to limited semantic modeling capability, high computational cost, and a large number of learnable parameters. 
To address this challenge, we suggest the first pre-trained low-rank tensor decomposition (PLTD) framework, which organically integrates the pre-trained large vision model into the classical tensor decomposition framework.
Beyond the shallow and untrained deep tensor decomposition, the suggested PLTD achieves an unprecedented balance among higher recovery fidelity, fewer learnable parameters, and smaller carbon footprint.
Specifically, PLTD factorizes the target tensor into a latent tensor and a learnable transform that maps the latent tensor back to the original data domain. The latent tensor consists of two indispensable and complementary terms, i.e., a fixed pre-trained latent tensor and a learnable low-rank latent tensor. 
The fixed pre-trained latent tensor is distilled from a pre-trained large vision model (i.e., DINOv3) to capture the common structure of the target tensor, while the learnable low-rank latent tensor characterizes the instance-specific structure of the target tensor.
To examine the potential of PLTD, we develop the corresponding multi-dimensional image recovery model and theoretically justify the advantages of this framework. Additionally, we discuss the connections between PLTD and classical tensor decomposition frameworks. 
Extensive experiments on multi-dimensional image recovery demonstrate that PLTD consistently achieves superior performance compared with state-of-the-art methods.
\end{abstract}


\section{Introduction}\label{sec1}
\IEEEPARstart{M}ulti-dimensional images (i.e., color images, multispectral images, and remote sensing images) are widely used in scientific and engineering applications. In practice, these images are often degraded by noise \cite{8908805, 10078018}, missing entries \cite{xia2026low, zhao2015bayesian, qin2024nonconvex}, or other forms of corruption during acquisition and transmission, which severely limits their practical applications. Since multi-dimensional images generally exhibit strong structural dependencies and redundancy, effectively exploiting their intrinsic structures is essential for reliable recovery.

Low-rank tensor decomposition has witnessed significant advances over the past few years and has been widely used for multi-dimensional image modeling. By exploiting the intrinsic low-rank structures of multi-dimensional images, these methods have achieved strong performance across a broad range of vision tasks, including image denoising \cite{lu2016tensor, he2020non, 9064895, peng2022exact, 6909886}, inpainting \cite{zheng2021fully,9938394, qin2022low, 9115254, feng2023multiplex, zhang2021low}, and other recovery tasks \cite{8478366, 10025664, 9714775, 9288702}. 
However, traditional methods neglect the rich common structure across different images and rely on handcrafted shallow low-rank structures to characterize the instance-specific structure of each target tensor independently. 
Consequently, they are insufficient to capture complex intrinsic dependencies in multi-dimensional images.
This limitation becomes particularly pronounced under severe corruption, where the limited reliable observations make it difficult for shallow low-rank structures to accurately characterize the underlying data structures, resulting in substantial degradation in recovery performance.
\begin{figure}[t]
	\centering
	
	\includegraphics[width=\linewidth]{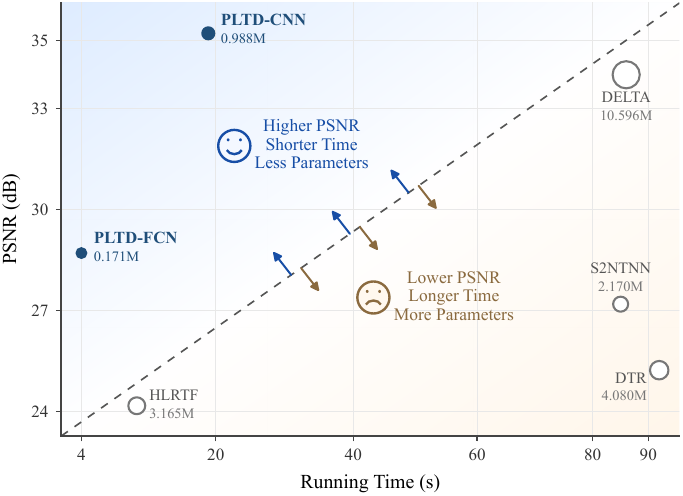}
	
	\caption{Comparison of the balance among recovery accuracy, learnable parameters, and running time. Each point represents the running time and recovery PSNR of a method, with the adjacent annotation indicating the number of learnable parameters. PLTD-FCN and PLTD-CNN achieve a more favorable trade-off among recovery accuracy, the number of learnable parameters, and running time.}
	\label{fig:motivation}
\end{figure}

Beyond handcrafted low-rank structures, deep tensor decomposition methods have attracted increasing attention in recent years \cite{10601492, 11231348, 10652890, luo2022hlrtf,11045408, 11271184, fan2021multi}. Instead of relying solely on shallow low-rank structures, these methods leverage neural networks to model tensor factor \cite{luo2023low, fan2021multi, han2024nested}, nonlinear transform \cite{luo2022hlrtf, 11231348, 9780890}, or latent tensor \cite{11045408, 10652890, zeng2026gaussian}. Recent studies have shown that deep neural networks can better characterize the complex nonlinear structures and often achieve stronger recovery performance than traditional tensor decomposition methods. 
However, existing deep tensor decomposition methods also neglect the common structure shared by different images and still face the following challenges. 
First, existing deep tensor decomposition methods typically employ sophisticated neural networks to learn the instance-specific structure of each target tensor from scratch, which inevitably result in computationally expensive and time-consuming training. 
Second, the large number of parameters makes model optimization more challenging and highly sensitive to network architectures and hyperparameter settings.

These limitations indicate that improving tensor modeling capability should not focus excessively on learning the instance-specific structure of each target tensor from scratch. 
Instead, a more potential direction is to introduce the common structure into tensor decomposition.
Recent advances in large vision models \cite{radford2021learning, kirillov2023segment, caron2021emerging, simeoni2025dinov3} make this strategy particularly appealing. By training on large-scale image datasets, these models have shown a remarkable ability to extract favorable semantic representations that generalize well to various visual tasks \cite{xiao2023dive, pang2024hir, cheng2024transfer, yi2024text}. 
For tensor decomposition, these  semantic representations inherit the generalizable common structure from large-scale images, providing informative priors for multi-dimensional image modeling. 
Nevertheless, how to effectively incorporate the pre-trained large vision model  into low-rank tensor decomposition remains underexplored. 

In this paper, we introduce pre-training into tensor decomposition for the first time and propose a \textbf{P}re-trained \textbf{L}ow-rank \textbf{T}ensor \textbf{D}ecomposition (PLTD) framework. 
The proposed PLTD factorizes the target tensor into a latent tensor and a learnable transform.
Specifically, the latent tensor consists of two indispensable and complementary terms, i.e., a pre-trained latent tensor extracted from the frozen DINOv3 \cite{simeoni2025dinov3} and a learnable low-rank latent tensor, while the transform is implemented by a lightweight neural network.
The pre-trained latent tensor provides the common structure inherited from DINOv3, while the low-rank latent tensor captures the instance-specific structure of the target tensor.
Additionally, the pre-trained latent tensor contains rich semantic and structural information, enabling expressive tensor modeling with only a few learnable parameters.
As shown in Fig.~\ref{fig:motivation}, the proposed PLTD achieves remarkable recovery performance with significantly fewer learnable parameters and less running time.
Our contributions can be summarized as follows:
\begin{enumerate}
	
	\item To the best of our knowledge, we introduce pre-training into tensor decomposition for the first time and propose PLTD. Beyond shallow tensor decomposition and untrained deep tensor decomposition, the proposed PLTD achieves an unprecedented balance among higher recovery fidelity, fewer learnable parameters, and smaller carbon footprint.

	\item Building on PLTD, we suggest the corresponding multi-dimensional image recovery model and theoretically justify the advantages of this framework. Additionally, we discuss the connections between PLTD and classical tensor decomposition frameworks.
	
	\item Extensive experiments show that PLTD achieves state-of-the-art recovery performance across multiple benchmark datasets. Moreover, PLTD demonstrates strong generalization capability across challenging cross-domain images and real-world images. 
	
\end{enumerate}

The remainder of this paper is organized as follows. Section~\ref{sec2} reviews related work on shallow and untrained deep tensor decomposition. Section~\ref{sec3} introduces the notation and preliminaries. Section~\ref{sec4} presents the proposed PLTD framework and the corresponding multi-dimensional image recovery model. Section~\ref{sec5} presents the experimental results. Section~\ref{sec6} discusses the experimental findings and provides insights. Finally, Section~\ref{sec7} concludes the paper.

\section{Related Work}\label{sec2}

\subsection{Shallow Tensor Decomposition Methods}\label{sub1sec2}
Shallow tensor decomposition methods mainly rely on handcrafted shallow low-rank structures to characterize the target tensor. Classical methods include CANDECOMP/PARAFAC (CP) decomposition \cite{hitchcock1927expression}, which represents a tensor as a sum of rank-one components, and Tucker decomposition \cite{6138863}, which factorizes a tensor into a core tensor and a set of matrix factors. Beyond these classical decomposition methods, subsequent methods can be broadly categorized into transform-based decomposition and tensor network decomposition. 

Transform-based decomposition methods capture low-rank structures by transforming the target tensor into an appropriate domain. As a representative framework, tensor singular value decomposition (t-SVD) \cite{hao2013facial} leverages the discrete Fourier transform (DFT) along the third mode and performs matrix factorization in the DFT domain. Subsequent studies explored flexible transforms to obtain more compact low-rank representations. For example, Lu et al. \cite{lu2019low} adopted the discrete cosine transform (DCT), while Song et al. \cite{song2020robust} extended the DFT transform to a general unitary transform. Moving beyond fixed transforms, Jiang et al. \cite{jiang2021dictionary} proposed DTNN, which learns a dictionary-based transform to better capture the intrinsic low-rank structure of the image.
Tensor network decomposition methods leverage different network topologies to characterize the low-rank structure of the target tensor. Representative methods include tensor train (TT) decomposition \cite{oseledets2011tensor}, tensor ring (TR) decomposition \cite{zhao2016tensor}, and fully-connected tensor network (FCTN) decomposition \cite{zheng2021fully}. These methods have achieved promising performance in multi-dimensional image recovery and have demonstrated strong modeling capability for other higher-order data \cite{xu2025fold, 7859390, zhou2019bayesian, yuan2019tensor}.

Nevertheless, shallow tensor decomposition methods are fundamentally built upon shallow low-rank structures to characterize the instance-specific structure of each target tensor, limiting their ability to characterize complex and nonlinear structures in multi-dimensional image.

\subsection{Untrained Deep Tensor Decomposition Methods}\label{sub2sec2}
Unlike shallow tensor decomposition methods that rely on handcrafted low-rank structures, untrained deep tensor decomposition leverages neural networks to achieve more expressive tensor modeling. 
For classical tensor decomposition, Fan et al. \cite{fan2021multi} introduced neural networks into Tucker decomposition to characterize the factor matrices, thereby capturing complex nonlinear dependencies along different tensor modes. Luo et al. \cite{luo2023low} further introduced implicit neural representations (INRs) to model the factor matrices, enabling a continuous characterization of the underlying data structure.

For transform-based decomposition, untrained deep methods typically replace these transforms with nonlinear neural networks. For example, Luo et al. \cite{9780890, luo2022hlrtf} replaced the DFT with a fully connected network (FCN) to better capture the low-rank structure of the latent tensor. Yang et al. \cite{11231348} further introduced the Transformer \cite{3295222.3295349} as the nonlinear transform to capture long-range dependencies and diverse patterns across multiple subspaces of the data.
The development of transform-based tensor decomposition has mainly focused on designing the transform, whereas the structure of
the latent tensor remains relatively underexplored. Recent studies have leveraged neural networks to more effectively characterize the latent tensor. For example, Zeng et al. \cite{zeng2026gaussian} employ Gaussian splatting to characterize the latent tensor, improving the representation of localized high-frequency details. Saragadam et al. \cite{10652890} use 2D convolutional networks to generate the latent tensor, enabling nonlinear modeling of complex spatial dependencies.
In parallel, recent studies have introduced neural networks into tensor network decomposition. By incorporating neural networks into handcrafted tensor network topologies, these methods \cite{Xu_2026_CVPR, han2024nested} enhance the flexibility of traditional tensor network models in characterizing complex dependencies of higher-order tensor.

However, these methods neglect the common structure across multi-dimensional images. They rely on complex neural networks to learn the instance-specific structure of each target tensor from scratch, resulting in computationally expensive and time-consuming training.
Additionally, increasingly complex network architectures further limit the practicality of deep tensor decomposition.

\section{Notation and Preliminaries}\label{sec3}
Throughout this paper, scalars, vectors, matrices, and tensors are denoted by $x$, $\mat{x}$, $\mat{X}$, and $\mathcal{X}$, respectively. For an $N$th-order tensor $\mathcal{X} \in \mathbb{R}^{I_1 \times I_2 \times \cdots \times I_N}$, $\mathcal{X}(i_1,i_2,\ldots,i_N)$ denotes its $(i_1,i_2,\ldots,i_N)$th entry. For a third-order tensor $\mathcal{X}\in\mathbb{R}^{n_1\times n_2\times n_3}$, a frontal slice is defined by fixing the third-mode index, i.e., $\mathcal{X}(:,:,i)$, where $i=1,2,\ldots,n_3$. For an $N$th-order tensor $\mathcal{X}$, its Frobenius norm is defined as
$\|\mathcal{X}\|_F := ( \sum_{i_1=1}^{I_1} \cdots \sum_{i_N=1}^{I_N} \mathcal X(i_1,\ldots,i_N)^2 )^{1/2}$.
Fibers are the higher-order analogue of matrix rows and columns. A fiber is defined by fixing every index but one \cite{kolda2009tensor}. For a tensor $\mathcal{X} \in \mathbb{R}^{I_1 \times I_2 \times \cdots \times I_N}$, the mode-$n$ fiber vectors are denoted by
$ \mathcal{X}(i_1, i_2, \cdots, i_{n-1}, :, i_{n+1}, \cdots, i_N) \in \mathbb{R}^{I_n}. $

\begin{definition}[Mode-$n$ Unfolding \cite{kolda2009tensor}]
	For an $N$th-order tensor $\mathcal{X}\in\mathbb{R}^{I_1\times I_2\times\cdots\times I_N}$, the mode-$n$ unfolding $\mat{X}_{(n)}$ is obtained by organizing all mode-$n$ fibers of $\mathcal{X}$ as columns, which is defined as follows:  
	\[ \mat{X}_{(n)}(i_n, j) = \mathcal{X}(i_1, i_2, \cdots, i_n, \cdots, i_N), \]
	where $j = 1 + \sum_{\substack{k=1 \\ k \neq n}}^{N} (i_k-1) \prod_{\substack{m=1 \\ m \neq n}}^{k-1} I_m$, and $n=1,2,\cdots,N$. The inverse process of mode-$n$ unfolding is denoted as $\fold_n$, i.e., $\mathcal{X} = \fold_n(\mat{X}_{(n)})$.
\end{definition}

\begin{definition}[Mode-$n$ Product \cite{kolda2009tensor}] \label{df:mode-n product}
	Given an $N$th-order tensor $\mathcal{G} \in \mathbb{R}^{I_1 \times I_2 \times \cdots \times I_N}$ and a matrix $\mat{U} \in \mathbb{R}^{J \times I_n}$, the mode-$n$ product $\mathcal{G} \times_n \mat{U} \in \mathbb{R}^{I_1 \times \cdots \times I_{n-1} \times J \times I_{n+1} \times \cdots \times I_N}$ is defined as
	\[ \mathcal{G} \times_n \mat{U} = \fold_n (\mat{U} \mat{G}_{(n)}), \]
	where $n=1,2,\cdots,N$.
\end{definition}

\begin{definition}[Outer Product \cite{8497054}]\label{def:outer-product}
	The outer product of a matrix $\mathbf{E}\in\mathbb{R}^{n_1\times n_2}$ and a vector $\mathbf{c}\in\mathbb{R}^{n_3}$ is defined element-wise as
	\begin{equation}\nonumber
		(\mathbf E\circ\mathbf c)(i,j,k)
		=
		\mathbf E(i,j)\mathbf c(k).
	\end{equation}
\end{definition}

\begin{definition}[CP Decomposition \cite{kolda2009tensor}]
	For a third-order tensor $\mathcal{X}\in\mathbb{R}^{n_1\times n_2\times n_3}$, the CP decomposition factorizes $\mathcal{X}$ into a sum of component rank-one tensors, i.e.,
	\begin{equation}
		\mathcal{X}
		=
		\sum_{j=1}^{J}
		\mathbf{u}_{j}^{(1)}
		\circ
		\mathbf{u}_{j}^{(2)}
		\circ
		\mathbf{u}_{j}^{(3)},
	\end{equation}
	where
	$\mathbf{u}_{j}^{(i)}\in\mathbb{R}^{n_i}$,
	$i=1,2,3$, are the factor vectors, and
	$J$ denotes the number of rank-one components.
\end{definition}

\begin{definition}[Tucker Decomposition \cite{kolda2009tensor}]
	For a third-order tensor $\mathcal{X}\in\mathbb{R}^{n_1\times n_2\times n_3}$, the Tucker decomposition factorizes $\mathcal{X}$ as
	\begin{equation}
		\mathcal{X}
		=
		\mathcal{G}
		\times_1 \mathbf{U}^{(1)}
		\times_2 \mathbf{U}^{(2)}
		\times_3 \mathbf{U}^{(3)},
	\end{equation}
	where
	$\mathcal{G}\in\mathbb{R}^{r_1\times r_2\times r_3}$
	is the core tensor and
	$\mathbf{U}^{(i)}\in\mathbb{R}^{n_i\times r_i}$,
	$i=1,2,3$, are the factor matrices.
\end{definition}

\begin{definition}[Block Term Decomposition \cite{doi:10.1137/070690729}]
	For a third-order tensor 
	$\mathcal{X}\in\mathbb{R}^{n_1\times n_2\times n_3}$, 
	the block term decomposition factorizes $\mathcal{X}$ into a sum of block terms, i.e.,
	\begin{equation}
		\mathcal{X}
		=
		\sum_{j=1}^{J}
		\left(\mathbf{A}_j\mathbf{B}_j^{\top}\right)
		\circ \mathbf{c}_j,
	\end{equation}
	where 
	$\mathbf{A}_j\in\mathbb{R}^{n_1\times L_j}$,
	$\mathbf{B}_j\in\mathbb{R}^{n_2\times L_j}$, and
	$\mathbf{c}_j\in\mathbb{R}^{n_3}$,
	$j=1,\ldots,J$, are the factor matrices and factor vectors, respectively,
	and $J$ denotes the number of block terms.
\end{definition}

\section{Proposed Method}\label{sec4}
In this section, we first introduce the proposed pre-trained low-rank tensor decomposition framework, i.e., PLTD. Then, we discuss the relationship between PLTD and two representative tensor decomposition paradigms: classical shallow tensor decomposition and untrained deep tensor decomposition. Finally, we give a PLTD-based multi-dimensional image recovery model and establish the recovery error bound between the recovered tensor by our model and the underlying tensor.
\begin{figure*}[t]
	\centering
	\includegraphics[width=\linewidth]{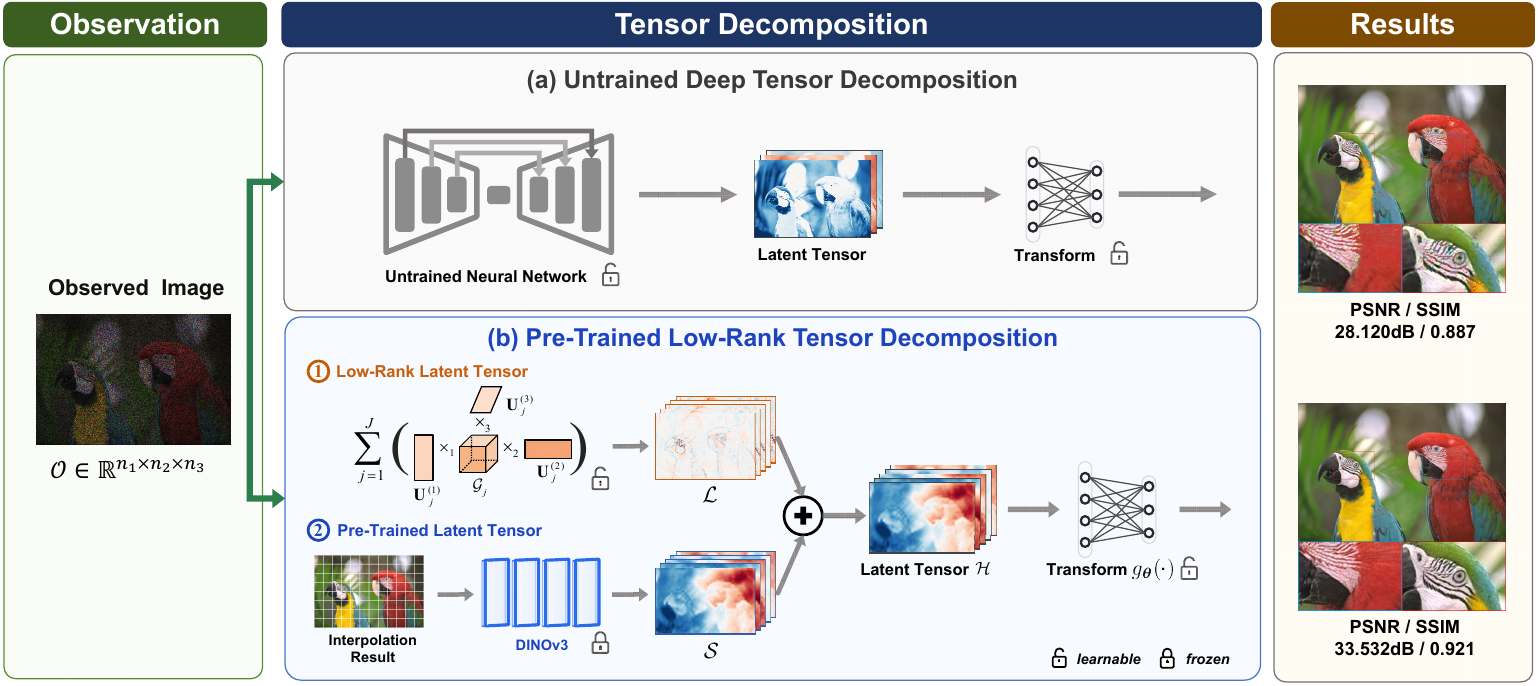}
	\caption{Flowchart of our PLTD framework. (a) Untrained deep tensor decomposition methods learn a heavy neural network from scratch to construct the recovery tensor. (b) PLTD organically integrates the pre-trained latent tensor with learnable low-rank tensor to achieve higher recovery fidelity, fewer learnable parameters, and smaller carbon footprint.}
	\label{fig: flowchart}
\end{figure*}
\subsection{The Proposed PLTD}

Transform-based tensor decomposition leverages a latent tensor and a transform to characterize the instance-specific structure of multi-dimensional images.
Shallow tensor decomposition relies on handcrafted low-rank structures to model the latent tensor and transform, limiting its representation capability.
Untrained deep tensor decomposition improves representation capability with neural networks but learns each target tensor from scratch, resulting in computationally expensive and time-consuming training.
To overcome these limitations, we introduce the common structure inherited from pre-trained large vision model into tensor decomposition, thereby achieving an unprecedented balance among higher recovery fidelity, fewer learnable parameters, and smaller carbon footprint.

Given a multi-dimensional image $\mathcal{X}\in\mathbb{R}^{n_1\times n_2\times n_3}$, PLTD factorizes it into a latent tensor $\mathcal{H}\in\mathbb{R}^{n_1\times n_2\times \hat n_3}$ and a lightweight learnable transform $g_{\boldsymbol\theta}(\cdot)$, where $\boldsymbol\theta$ represents the network parameters. Specifically, the latent tensor $\mathcal{H}$ consists of a pre-trained latent tensor $\mathcal{S}\in\mathbb{R}^{n_1\times n_2\times \hat n_3}$ extracted from a frozen DINOv3 and a learnable low-rank latent tensor $\mathcal{L}\in\mathbb{R}^{n_1\times n_2\times \hat n_3}$.
The pre-trained latent tensor $\mathcal{S}$ captures the common structure across different multi-dimensional images, while the learnable low-rank latent tensor $\mathcal{L}$ characterizes the instance-specific structure of the target tensor.
The flowchart of our PLTD framework is shown in Fig.~\ref{fig: flowchart}, and its mathematical formulation is given as follows:
\begin{equation}
	\mathcal{X}
	=
	g_{\boldsymbol\theta}
	\big(
	\sum_{j=1}^{J} \mathcal{G}_{j} \times_{1} \mathbf{U}_{j}^{(1)} \times_{2} \mathbf{U}_{j}^{(2)} \times_{3} \mathbf{U}_{j}^{(3)}
	+\mathcal{S}
	\big),
	\label{eq:PLTD_main}
\end{equation}
where $\mathcal{G}_{j}\in\mathbb{R}^{r_{1,j}\times r_{2,j}\times r_{3,j}}$, $\mathbf{U}_{j}^{(1)}\in\mathbb{R}^{n_1\times r_{1,j}}$, $\mathbf{U}_{j}^{(2)}\in\mathbb{R}^{n_2\times r_{2,j}}$, and $\mathbf{U}_{j}^{(3)}\in\mathbb{R}^{\hat n_3\times r_{3,j}}$. The summation term in Eq.~\eqref{eq:PLTD_main} provides a general low-rank tensor decomposition framework for $\mathcal L$, with several representative tensor decompositions as special cases:

\textit{(i) Slice-Wise Matrix Decomposition.}
By setting $J=\hat n_3$, $r_{1,j}=r_{2,j}=r_j$, $r_{3,j}=1$,
$\mathcal G_j(:,:,1)=\mathbf I_{r_j}$,
$\mathbf U_j^{(1)}=\mathbf A_j \in\mathbb R^{n_1\times r_j}$,
$\mathbf U_j^{(2)}=\mathbf B_j^\top\in\mathbb R^{n_2\times r_j}$, and
$\mathbf U_j^{(3)}=\mathbf e_j\in\mathbb R^{\hat n_3}$,
where $\mathbf e_j$ denotes the $j$th canonical basis vector,
the low-rank latent tensor $\mathcal L$ becomes
$\sum_{j=1}^{\hat n_3}(\mathbf A_j\mathbf B_j)\circ\mathbf e_j$,
which corresponds to slice-wise matrix decomposition.

\textit{(ii) CP Decomposition.}
By setting $r_{1,j}=r_{2,j}=r_{3,j}=1$ for $j=1,\ldots,J$,
$\mathcal G_j=g_j\in\mathbb{R}$,
$\mathbf U_j^{(1)}=\mathbf u_j^{(1)}\in\mathbb{R}^{n_1}$,
$\mathbf U_j^{(2)}=\mathbf u_j^{(2)}\in\mathbb{R}^{n_2}$, and
$\mathbf U_j^{(3)}=\mathbf u_j^{(3)}\in\mathbb{R}^{\hat n_3}$,
the low-rank latent tensor $\mathcal L$ becomes
$\sum_{j=1}^{J}\mathbf u_j^{(1)}\circ\mathbf u_j^{(2)}\circ\mathbf u_j^{(3)}$,
which corresponds to the CP decomposition.

\textit{(iii) Tucker Decomposition.}
By setting $J=1$, the low-rank latent tensor $\mathcal L$ becomes
$\mathcal{G}\times_{1}\mathbf U^{(1)}
\times_{2}\mathbf U^{(2)}
\times_{3}\mathbf U^{(3)}$,
which is exactly the Tucker decomposition.

\textit{(iv) Block Term Decomposition.}
By setting $J=1$, $r_{1,1}=r_{2,1}=r$, $r_{3,1}=1$,
$\mathcal G_1(:,:,1)=\mathbf I_r$,
$\mathbf U_1^{(1)}=\mathbf A \in\mathbb{R}^{n_1\times r}$,
$\mathbf U_1^{(2)}=\mathbf B^\top \in\mathbb{R}^{n_2\times r}$, and
$\mathbf U_1^{(3)}=\mathbf c \in\mathbb{R}^{\hat n_3}$,
the low-rank latent tensor $\mathcal L$ becomes
$(\mathbf A\mathbf B)\circ\mathbf c$, which corresponds to a single block term in the block term decomposition.

Unless otherwise specified, we adopt this low-rank decomposition in PLTD to characterize the instance-specific structure of the target tensor. The influence of other low-rank structures of the learnable low-rank latent tensor will be further discussed in the Discussion~\ref{dis: extension}.

As shown in Eq.~\eqref{eq:PLTD_main}, PLTD comprises two components: the latent tensor (i.e., $\mathcal{H} = \mathcal{L} + \mathcal{S}$) and a learnable transform (i.e., $g_{\boldsymbol{\theta}} (\cdot)$). The former combines the pre-trained latent tensor with the learnable low-rank latent tensor, while the latter transforms the latent tensor back to the original domain. We now introduce these two components in detail. 
\paragraph{Latent Tensor}
\begin{figure}[htbp]
	\centering
	\setlength{\tabcolsep}{5pt}
	\begin{tabular}{ccc}
		\includegraphics[width=2.6cm]{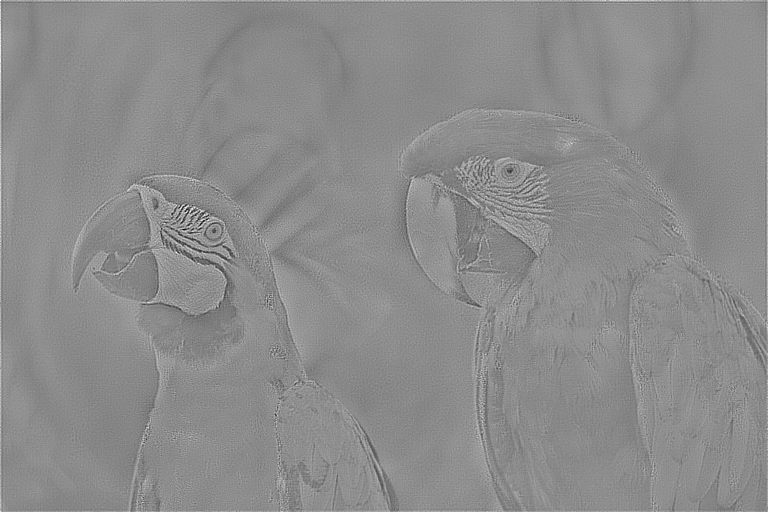} &
		\includegraphics[width=2.6cm]{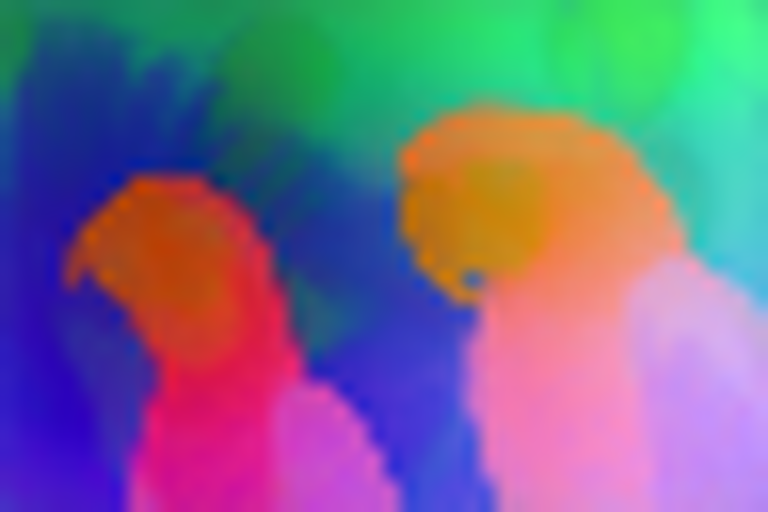} &
		\includegraphics[width=2.6cm]{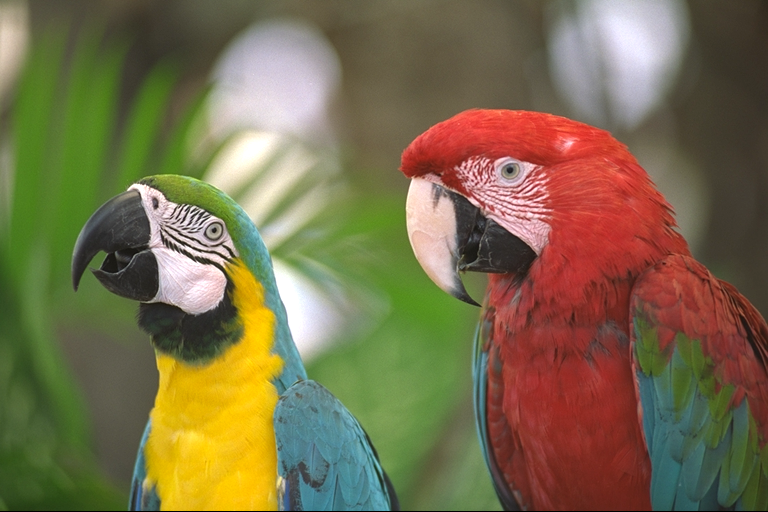} \\
		
		\includegraphics[width=2.6cm]{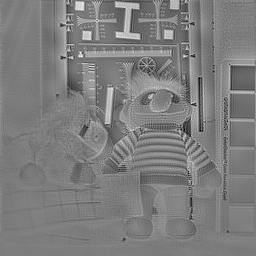} &
		\includegraphics[width=2.6cm]{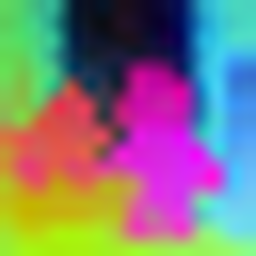} &
		\includegraphics[width=2.6cm]{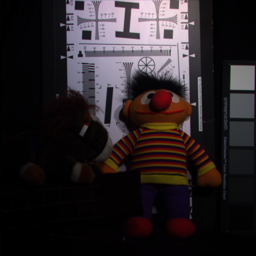} \\
		
		\includegraphics[width=2.6cm]{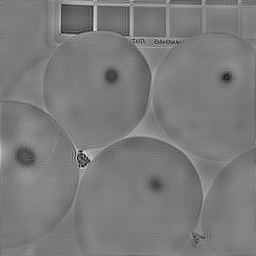} &
		\includegraphics[width=2.6cm]{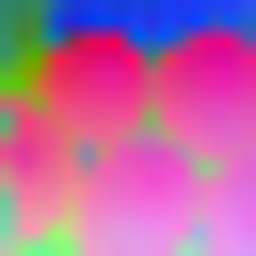} &
		\includegraphics[width=2.6cm]{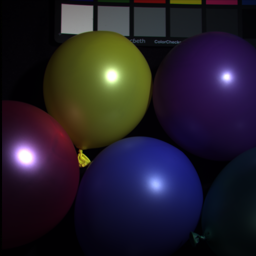} \\
		
		{\scriptsize
			\begin{tabular}[c]{@{}c@{}}
				Low-Rank Latent Tensor\\
				(Instance-Specific Structure)
		\end{tabular}}
		&
		{\scriptsize
			\begin{tabular}[c]{@{}c@{}}
				Pre-Trained Latent Tensor\\
				(Common Structure)
		\end{tabular}}
		&
		{\scriptsize
			\begin{tabular}[c]{@{}c@{}}
				Ground Truth
		\end{tabular}}
		\\
	\end{tabular}
	\caption{Visualization of the learnable low-rank latent tensor and the pre-trained latent tensor. The two latent tensors are complementary and jointly indispensable for latent tensor modeling. The fixed pre-trained latent tensor captures the common structure across different multi-dimensional images, while the learnable low-rank latent tensor characterizes the instance-specific structure of the target tensor.
	}
	\label{fig:fine-tune}
\end{figure}
Unlike shallow tensor decomposition methods, which model the latent tensor independently through handcrafted low-rank structures, and untrained deep tensor decomposition methods, which rely on neural networks to learn the latent tensor from scratch, 
PLTD constructs the latent tensor using two complementary and indispensable components: a pre-trained latent tensor $\mathcal{S}$ and a learnable low-rank latent tensor $\mathcal{L}$.
Given the coarse initialization obtained by interpolating the observation, we construct the pre-trained latent tensor $\mathcal{S}$ by feeding the coarse initialization into a frozen DINOv3 model. 
The pre-trained latent tensor $\mathcal{S}$ inherits rich structural and semantic information from DINOv3, providing a reliable common structure for latent tensor modeling.

Although $\mathcal{S}$ offers a meaningful common structure, it is not inherently tailored to the target tensor. Directly using $\mathcal{S}$ may therefore lead to misalignment with the desired latent tensor. 
To bridge the gap between $\mathcal{S}$ and the desired latent tensor, PLTD introduces the learnable low-rank tensor $\mathcal{L}$ to characterize the instance-specific structure. 
As illustrated in Fig.~\ref{fig:fine-tune}, the pre-trained latent tensor provides the dominant structure of the desired latent tensor, while the low-rank latent tensor supplements instance-specific structure for latent space alignment. Therefore, the fixed pre-trained latent tensor and the learnable low-rank latent tensor play complementary and indispensable roles in characterizing the latent tensor.

\paragraph{Learnable Transform}
To better capture the dependencies among the frontal slices of multi-dimensional image, the learnable transformation $g_{\boldsymbol{\theta}}(\cdot)$ in PLTD is implemented as a lightweight convolutional neural network, formulated as: 
\begin{equation}\nonumber
	g_{\boldsymbol{\theta}}(\mathcal{H})
	=
	\mathcal{W}_{L}\circledast
	\sigma\!\big(
	\mathcal{W}_{L-1}\circledast
	\cdots
	\sigma\!\left(
	\mathcal{W}_{1}\circledast\mathcal{H}
	\right)
	\big),
\end{equation}
where $\circledast$ denotes the convolution operation, $\sigma(\cdot)$ is the nonlinear activation function, and
$\boldsymbol{\theta}=\{\mathcal{W}_{l}\}_{l=1}^{L}$ denotes the learnable parameters of the convolutional neural network (CNN). This learnable transform enables nonlinear modeling of dependencies among frontal slices and flexible mapping from the latent tensor to the recovery tensor with only a small number of learnable parameters.

\emph{Remark 1:}
We further discuss the relationship between PLTD and existing classical tensor decomposition methods. In particular, shallow tensor decomposition and untrained deep tensor decomposition can be viewed as two special cases related to our framework.

(i) If the latent tensor excludes the pre-trained latent tensor $\mathcal{S}$, our PLTD degenerates into the shallow tensor decomposition. In this case, the latent tensor relies primarily on handcrafted low-rank structures, which are insufficient to characterize complex dependencies.

(ii) If the latent tensor excludes the learnable low-rank latent tensor $\mathcal{L}$ and the pre-trained large vision model is unfrozen, our PLTD degenerates into the untrained deep tensor decomposition \cite{11045408}. In this case, the latent tensor is mainly learned from scratch through a neural network, which usually leads to computationally expensive and time-consuming training.

Compared with shallow tensor decomposition, PLTD inherits meaningful common structure from a pre-trained large vision model to enhance the representation capability beyond handcrafted low-rank structures.   
Compared with untrained deep tensor decomposition, PLTD avoids learning heavy neural networks from scratch and instead learns only the lightweight low-rank tensor. 
By organically integrating the pre-trained large vision model into the classical tensor decomposition framework, our PLTD achieves an unprecedented balance among higher recovery fidelity, fewer learnable parameters, and smaller carbon footprint.

\subsection{PLTD-based Multi-Dimensional Image Recovery Model}
In this section, we develop a PLTD-based multi-dimensional image recovery model for real-world recovery tasks. In summary, PLTD consists of two key components: (i) a latent tensor that integrates a pre-trained latent tensor with a learnable low-rank latent tensor, and (ii) a lightweight learnable transform that maps the latent tensor back to the original domain. Specifically, we first introduce the recovery model and then present the solving algorithm. 
\paragraph{Recovery Model}
To evaluate the representational capability of PLTD, we formulate the PLTD-based multi-dimensional image recovery model as
\begin{equation}
	\begin{aligned}
		\min_{\boldsymbol\theta, \mathbf A, \mathbf B, \mathbf c} \quad & \frac{1}{2}\|\mathcal{P}_{\Omega}(\mathcal{X}-\mathcal{O})\|_F^2 \\
		\text{s.t.}\quad & \mathcal{X}=g_{\boldsymbol\theta}
		\big(
		(\mathbf A\mathbf B)\circ\mathbf c
		+\mathcal{S}
		\big),
	\end{aligned}
	\label{eq:model}
\end{equation}
where $\Omega$ is the observed index set, $\mathcal{P}_{\Omega}(\cdot)$ retains the entries in $\Omega$ and sets the others to zero, $\mathcal{O}\in\mathbb{R}^{n_1\times n_2\times n_3}$ is an incomplete observation, $\mathcal{X}\in\mathbb{R}^{n_1\times n_2\times n_3}$ is the recovered tensor. $\mathbf A\in\mathbb{R}^{n_1\times r}$, $\mathbf B\in\mathbb{R}^{r\times n_2}$, $\mathbf c\in\mathbb{R}^{\hat n_3}$. The optimization variables are the network parameters $\boldsymbol{\theta}$, together with the low-rank tensor factors  $\{\mathbf{A},\mathbf{B},\mathbf{c}\}$, while the pre-trained latent tensor $\mathcal{S}$ remains frozen.

\paragraph{Solving Algorithm}
The model in Eq.~\eqref{eq:model} is highly nonconvex and nonlinear, making it challenging to solve. In this work, we adopt the popular adaptive moment estimation (Adam) \cite{kingma2014adam} algorithm to efficiently optimize the network parameters $\boldsymbol\theta$ and the low-rank tensor factors $\{\mathbf{A},\mathbf{B},\mathbf{c}\}$. After optimization, the recovered tensor $\mathcal{X}$ can be obtained by Eq.~\eqref{eq:PLTD_main}. Meanwhile, we set the maximum iteration number $t_{\max}$ as the stopping criterion for the Adam-based optimization algorithm in our experiments. The pseudocode of the Adam-based optimization algorithm is presented in Algorithm~\ref{alg1}.
\begin{algorithm}[htbp]
	\caption{Solving algorithm for the proposed multi-dimensional image recovery model.}
	\label{alg1}
	\begin{algorithmic}[1]
		\STATE \textbf{Input:} The observed image $\mathcal{O}$, the observed index $\Omega$, and the factorization rank $r$ in low-rank tensor factors $\{\mathbf{A},\mathbf{B},\mathbf{c}\}$.
		\STATE \textbf{Initialization:} Initial low-rank tensor factors $\{\mathbf{A},\mathbf{B},\mathbf{c}\}$, CNN weights $\boldsymbol\theta$.
		\STATE Obtain the pre-trained latent tensor $\mathcal{S}$ from DINOv3.
		\FOR{$t = 1,2,\ldots,t_{\max}$}
		\STATE Compute the recovered image $\mathcal{X}$ (Eq.~\eqref{eq:PLTD_main});
		\STATE Compute the loss function of the model (Eq.~\eqref{eq:model});
		\STATE Update $\{ \mathbf{A},\mathbf{B},\mathbf{c}, \boldsymbol\theta \}$  using the Adam optimizer;
		\ENDFOR
		\STATE \textbf{Output:} The recovery result $\mathcal{X}$.
	\end{algorithmic}
\end{algorithm}

\subsection{Theoretical Error Bound}
To theoretically characterize the recovery performance of the PLTD-based recovery model, we establish an upper bound on the recovery error. 
Furthermore, our analysis reveals that introducing pre-trained common structure can lead to a tighter recovery error bound than that of untrained deep tensor decomposition under general conditions.

\begin{theorem}[Recovery Error Bound \cite{Vershynin_2018, fan2021multi}]
	\label{thm}
	
	Let $\hat{\mathcal{X}}\in\mathbb{R}^{n_1\times n_2\times n_3}$ be the underlying tensor. For any fixed pre-trained latent tensor
	$\mathcal S\in
	\mathbb{R}^{n_1\times n_2\times \hat n_3}$,
	let $\mathcal{X}$ be recovered from entries indexed by $\Omega$. Consider the model class
	\[
	\mathfrak{L}_{pre}
	=
	\left\{
	\mathcal{X}:
	\mathcal{X}
	=
	g_{\boldsymbol\theta}
	\big(
	(\mathbf A\mathbf B)\circ\mathbf c
	+\mathcal{S}
	\big)
	\right\},
	\]
	where $\mathbf{A}$, $\mathbf{B}$, and $\mathbf{c}$ are low-rank tensor factors,
	$\circ$ represents the outer product, and
	$g_{\boldsymbol\theta}$ is an $L$-layer convolutional neural network.
	Let
	$
	\mathcal Z_0=(\mathbf A\mathbf B)\circ\mathbf c+\mathcal S
	$
	and define
	$
		\mathcal Z_l
		=
		\sigma(\mathcal W_l\circledast\mathcal Z_{l-1}),
		\quad l=1,\ldots,L-1,
	$
	with
	$
	\mathcal X=\mathcal W_L\circledast\mathcal Z_{L-1}.
	$
	Here,
	$\mathcal W_l\in\mathbb R^{k\times k\times c_{l-1}\times c_l}$
	for $l=1,\ldots,L$, $k\in\mathbb N$ denotes the spatial kernel size,
	and $c_l$ denotes the number of channels at the $l$th layer, with
	$c_0=\hat n_3$ and $c_L=n_3$.
	
	Assume that
	$\mathbf A\in\mathbb R^{n_1\times r}$, $\mathbf B\in\mathbb R^{r\times n_2}$, $\mathbf c\in\mathbb R^{\hat n_3}$, 
	$
	\|\mathbf A\|_F\leq\alpha_A$, 
	$\|\mathbf B\|_F\leq\alpha_B$, 
	$\|\mathbf c\|_2\leq\alpha_c
	$.
	For positive constants $\{\beta_l\}_{l=1}^{L}$, the convolution kernels satisfy
	$
	\|\mathcal W_l\|_F\leq \frac{\beta_l}{k}
	$, $l=1,\ldots,L$. The activation function is $\eta$-Lipschitz continuous for some $\eta>0$, and the
	intermediate features satisfy
	$
	\|\mathcal Z_l\|_F\leq M_l,
	\quad l=0,\ldots,L-1,
	$
	for positive constants $\{M_l\}_{l=0}^{L-1}$. 
	Moreover, suppose that $\Omega$ is randomly sampled and
	$
	\|\hat{\mathcal X}\|_{\infty}\leq\delta
	$, $
	\sup_{\mathcal X\in\mathfrak{L}_{pre}}
	\|\mathcal X\|_{\infty}\leq\delta
	$.
	Then there exists a constant $\epsilon_0>0$ such that, for any $0<\epsilon\leq\epsilon_0$, with probability at least
	$1-2N^{-1}$,
	\begin{equation}
	\begin{aligned}
		\frac{\|\hat{\mathcal X}-\mathcal X\|_F}{\sqrt N}
		&\leq
		\frac{
			\|\mathcal{P}_\Omega
			(\hat{\mathcal X}-\mathcal X)\|_F
		}
		{\sqrt{|\Omega|}}
		+
		\frac{2\epsilon}{\sqrt{|\Omega|}}
		\\
		&+
		\left(
		\frac{
			8\delta^4
			\left(
			\mathcal C
			\log
			\left(
			\frac{C_0}{\epsilon}
			\right)
			+
			\log N
			\right)
		}
		{|\Omega|}
		\right)^{1/4},
	\end{aligned}
	\end{equation}
	where
	$
	\mathcal C
	=
	r(n_1+n_2)+\hat n_3+H_g
	$, 
	and
	$
	H_g
	=
	\sum_{l=1}^{L}
	c_lc_{l-1}k^2
	$
	denotes the number of learnable parameters of the convolutional network. Here $N=n_1n_2n_3$ and $C_0>0$ is a sufficiently large constant independent of $\epsilon$.
	
	Theorem~\ref{thm} characterizes the recovery performance of the PLTD-based recovery model. In particular, the recovery error bound depends on the effective model complexity $r(n_1+n_2)+\hat n_3+H_g$. Here, $r(n_1+n_2)+\hat n_3$ denotes the complexity of the latent tensor, while $H_g$ corresponds to the complexity of the learnable transform. Since PLTD employs a lightweight convolutional network for the transform, $H_g$ remains relatively small. Moreover, compared with untrained deep tensor decomposition methods that directly learn complex network parameters to model the latent tensor, PLTD only learns a low-rank tensor with $r(n_1+n_2)+\hat n_3$ parameters to capture the instance-specific structure of the target tensor. Therefore, the effective model complexity of PLTD is generally much smaller than that of most untrained deep tensor decomposition methods, leading to a tighter recovery error bound. The detailed proof is provided in Appendix~\ref{appendix}.
\end{theorem}

\section{Experiments}\label{sec5}
In this section, we evaluate the effectiveness and generalization capability of the proposed PLTD framework under four challenging multi-dimensional images: color images \footnote{\url{https://sipi.usc.edu/database/}}, multispectral images\footnote{\url{https://cave.cs.columbia.edu/repository/Multispectral}}, real-world remote sensing images\footnote{\url{https://earthexplorer.usgs.gov/}}, and magnetic resonance images\footnote{\url{https://brainweb.bic.mni.mcgill.ca/brainweb/selection_normal.html}}. We compare PLTD with representative state-of-the-art shallow tensor decomposition methods and untrained deep tensor decomposition methods. The shallow tensor decomposition methods include TNN~\cite{6909886} and MTTD~\cite{feng2023multiplex}, while the untrained deep tensor decomposition methods include HLRTF~\cite{luo2022hlrtf}, S2NTNN~\cite{9780890}, DTR~\cite{11045408}, and DELTA~\cite{11231348}. For a fair comparison, the hyperparameter settings of all competing methods follow those reported in their original papers or official implementations as closely as possible. All experiments were conducted on a desktop computer equipped with an Intel(R) Core(TM) Ultra 5 250K Plus CPU, an NVIDIA GeForce RTX 5070 Ti GPU, and 32 GB of RAM. For a fair comparison, our method and other Python-based methods were implemented in PyTorch 2.11.0 with CUDA 12.8. The MATLAB-based methods were run using MATLAB R2026a.
\paragraph*{Experimental Settings}
\begin{table*}[htbp]
	\caption{Quantitative metrics of the different methods on color images under the different sampling conditions.}
	\label{tab:rgbm}
	\centering
	\footnotesize
	\setlength{\tabcolsep}{1.2pt}
	\renewcommand{\arraystretch}{1.08}
	
	\begin{tabular*}{\linewidth}{@{\extracolsep{\fill}} c c cc cc cc cc cc cc cc}
		
		\toprule
		\multirow{2}{*}{Data} & \multirow{2}{*}{SR}
		& \multicolumn{2}{c}{TNN}
		& \multicolumn{2}{c}{HLRTF}
		& \multicolumn{2}{c}{S2NTNN}
		& \multicolumn{2}{c}{MTTD}
		& \multicolumn{2}{c}{DTR}
		& \multicolumn{2}{c}{DELTA}
		& \multicolumn{2}{c}{PLTD} \\
		\cmidrule(lr){3-4}\cmidrule(lr){5-6}\cmidrule(lr){7-8}
		\cmidrule(lr){9-10}\cmidrule(lr){11-12}\cmidrule(lr){13-14}
		\cmidrule(lr){15-16}
		& & PSNR & SSIM & PSNR & SSIM & PSNR & SSIM & PSNR & SSIM & PSNR & SSIM & PSNR & SSIM & PSNR & SSIM \\
		\midrule
		
		\multirow{4}{*}{\shortstack{\emph{Butterfly}\\(256$\times$256$\times$3)}}
		& $2\%$ & 11.587 & 0.105 & 11.515 & 0.107 & 12.848 & 0.174 & 13.843 & 0.311 & 11.575 & 0.160 & \underline{14.504} & \underline{0.340} & \textbf{15.357} & \textbf{0.387} \\
		& $5\%$ & 13.375 & 0.247 & 12.964 & 0.198 & 14.425 & 0.318 & \underline{16.245} & \underline{0.464} & 15.135 & 0.424 & 16.160 & 0.458 & \textbf{17.943} & \textbf{0.543} \\
		& $10\%$ & 15.197 & 0.381 & 14.852 & 0.339 & 16.270 & 0.466 & 18.674 & 0.603 & 18.482 & 0.583 & \underline{19.625} & \underline{0.637} & \textbf{22.282} & \textbf{0.774} \\
		& $15\%$ & 16.614 & 0.450 & 16.627 & 0.445 & 17.516 & 0.520 & \underline{20.591} & 0.694 & 20.550 & \underline{0.704} & 20.560 & 0.666 & \textbf{25.256} & \textbf{0.855} \\
		\midrule
		
		\multirow{4}{*}{\shortstack{\emph{Airplane}\\(256$\times$256$\times$3)}}
		& $2\%$ & 17.354 & 0.340 & 14.978 & 0.164 & 17.428 & 0.340 & \underline{18.142} & \underline{0.514} & 17.321 & 0.389 & 18.110 & 0.425 & \textbf{18.853} & \textbf{0.515} \\
		& $5\%$ & 19.383 & 0.497 & 17.150 & 0.295 & 19.364 & 0.512 & \underline{20.138} & \underline{0.594} & 19.306 & \underline{0.594} & 19.795 & 0.504 & \textbf{21.156} & \textbf{0.659} \\
		& $10\%$ & 20.852 & 0.588 & 19.376 & 0.468 & 20.904 & 0.609 & 22.054 & 0.688 & 21.249 & \underline{0.730} & \underline{22.322} & 0.656 & \textbf{23.475} & \textbf{0.787} \\
		& $15\%$ & 22.168 & 0.641 & 21.294 & 0.577 & 22.120 & 0.675 & \underline{23.645} & \underline{0.756} & 23.184 & 0.733 & 23.390 & 0.722 & \textbf{25.637} & \textbf{0.848} \\
		\midrule
		
		\multirow{4}{*}{\shortstack{\emph{Peppers}\\(256$\times$256$\times$3)}}
		& $2\%$ & 12.135 & 0.110 & 12.429 & 0.105 & 13.673 & 0.145 & \underline{17.185} & \textbf{0.450} & 12.267 & 0.219 & 16.930 & 0.314 & \textbf{17.405} & \underline{0.372} \\
		& $5\%$ & 13.720 & 0.158 & 14.581 & 0.166 & 14.446 & 0.208 & \underline{19.667} & \underline{0.547} & 16.132 & 0.393 & 19.152 & 0.440 & \textbf{20.548} & \textbf{0.566} \\
		& $10\%$ & 15.951 & 0.220 & 17.273 & 0.300 & 17.111 & 0.317 & 21.458 & 0.636 & 20.342 & 0.526 & \underline{21.977} & \underline{0.644} & \textbf{23.240} & \textbf{0.704} \\
		& $15\%$ & 17.855 & 0.312 & 18.829 & 0.402 & 18.949 & 0.432 & 22.825 & 0.695 & \underline{22.951} & \underline{0.712} & 22.572 & 0.658 & \textbf{24.870} & \textbf{0.774} \\
		\midrule
		
		\multirow{4}{*}{\shortstack{\emph{House}\\(256$\times$256$\times$3)}}
		& $2\%$ & 15.046 & 0.181 & 13.509 & 0.118 & 15.455 & 0.208 & 18.997 & \textbf{0.541} & 17.555 & \underline{0.537} & \underline{19.591} & 0.424 & \textbf{20.099} & 0.489 \\
		& $5\%$ & 16.881 & 0.271 & 16.914 & 0.229 & 18.122 & 0.295 & 21.878 & \textbf{0.622} & 19.633 & 0.513 & \underline{22.688} & 0.559 & \textbf{22.689} & \underline{0.579} \\
		& $10\%$ & 18.825 & 0.346 & 20.062 & 0.412 & 20.483 & 0.481 & \underline{24.512} & \underline{0.705} & 23.053 & 0.603 & 24.435 & 0.646 & \textbf{25.318} & \textbf{0.718} \\
		& $15\%$ & 20.832 & 0.447 & 22.175 & 0.541 & 22.322 & 0.586 & \underline{26.428} & \underline{0.761} & 25.605 & 0.758 & 25.561 & 0.700 & \textbf{27.493} & \textbf{0.786} \\
		\bottomrule
	\end{tabular*}
	
	\vspace{2mm}
	The \textbf{best} and \underline{second-best} results are highlighted.
\end{table*}
\begin{figure*}[htbp]
	\footnotesize
	\begin{tabular*}{\linewidth}{CCCCCCCCC@{}}
		Observation & TNN & HLRTF & S2NTNN & MTTD & DTR & DELTA  & PLTD & Ground Truth \\
		\includegraphics[width=0.105\linewidth]{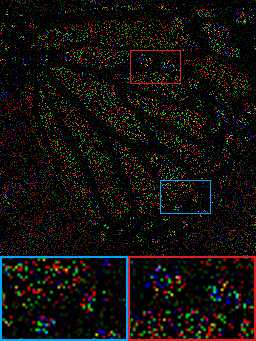} &
		\includegraphics[width=0.105\linewidth]{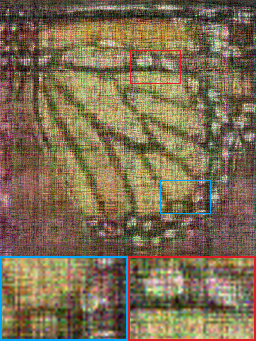} &
		\includegraphics[width=0.105\linewidth]{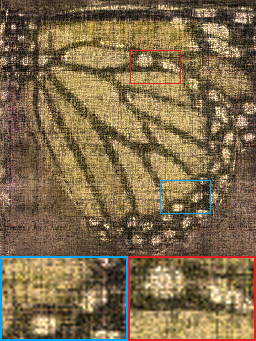} &
		\includegraphics[width=0.105\linewidth]{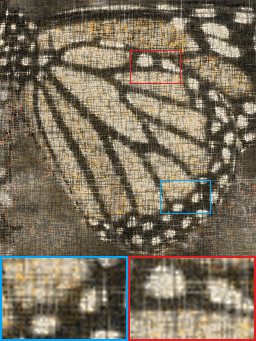} &
		\includegraphics[width=0.105\linewidth]{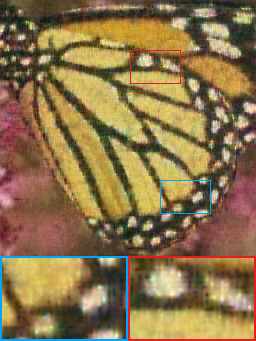} &
		\includegraphics[width=0.105\linewidth]{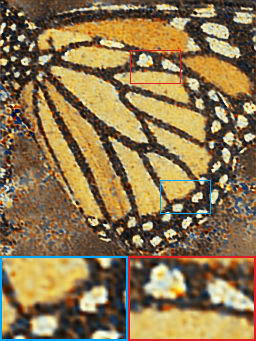} &
		\includegraphics[width=0.105\linewidth]{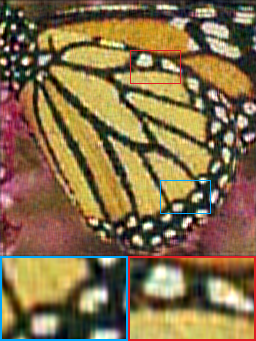} &
		\includegraphics[width=0.105\linewidth]{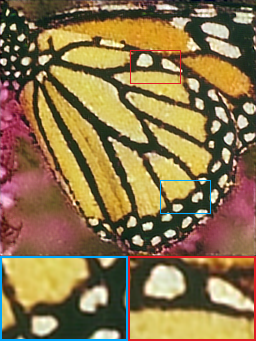} &
		\includegraphics[width=0.105\linewidth]{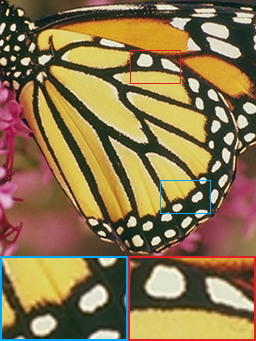} \\
		
		\includegraphics[width=0.105\linewidth]{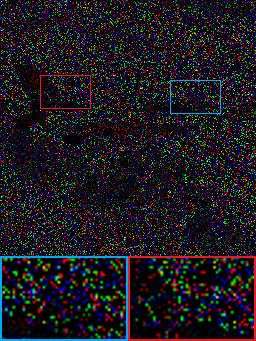} &
		\includegraphics[width=0.105\linewidth]{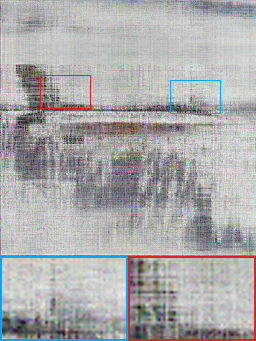} &
		\includegraphics[width=0.105\linewidth]{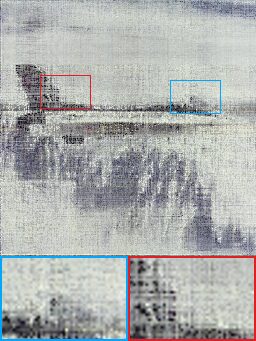} &
		\includegraphics[width=0.105\linewidth]{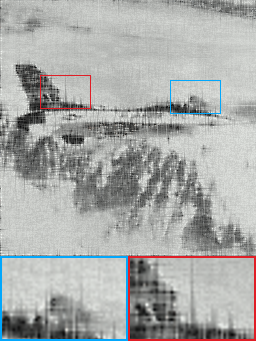} &
		\includegraphics[width=0.105\linewidth]{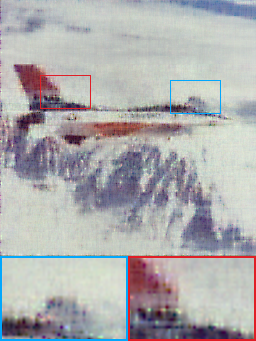} &
		\includegraphics[width=0.105\linewidth]{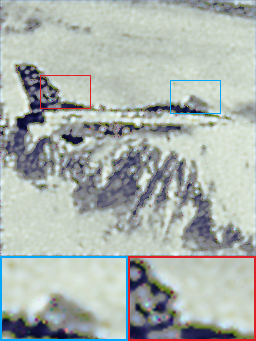} &
		\includegraphics[width=0.105\linewidth]{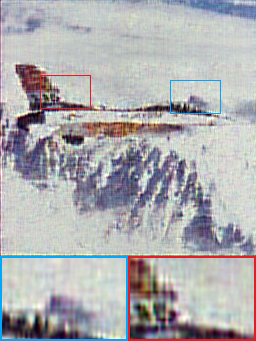} &
		\includegraphics[width=0.105\linewidth]{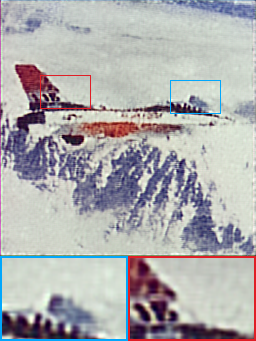} &
		\includegraphics[width=0.105\linewidth]{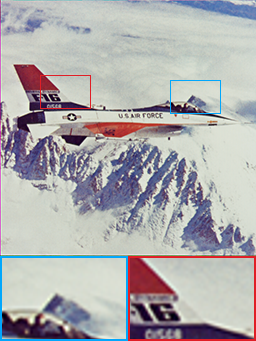} \\
		
		\includegraphics[width=0.105\linewidth]{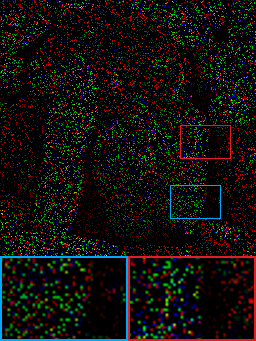} &
		\includegraphics[width=0.105\linewidth]{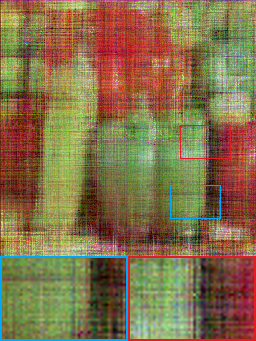} &
		\includegraphics[width=0.105\linewidth]{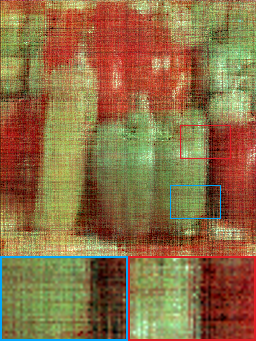} &
		\includegraphics[width=0.105\linewidth]{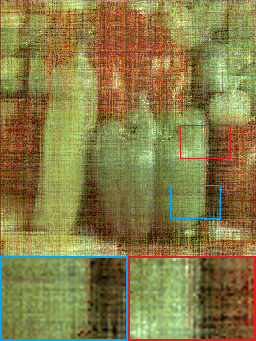} &
		\includegraphics[width=0.105\linewidth]{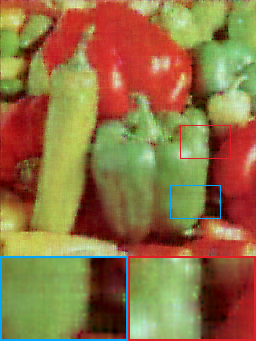} &
		\includegraphics[width=0.105\linewidth]{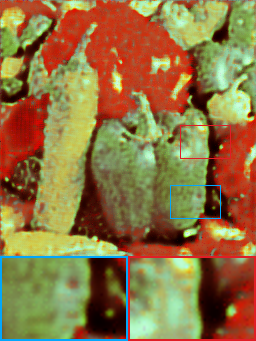} &
		\includegraphics[width=0.105\linewidth]{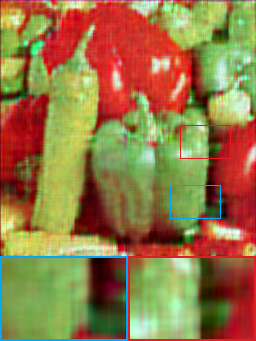} &
		\includegraphics[width=0.105\linewidth]{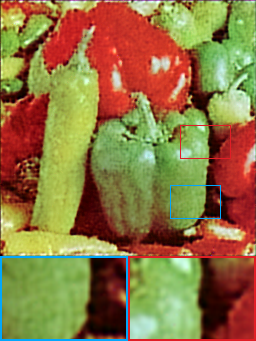} &
		\includegraphics[width=0.105\linewidth]{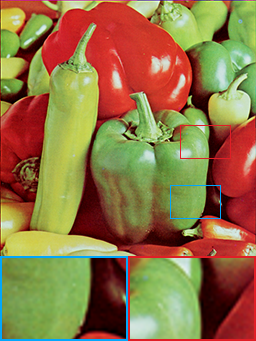} \\
	\end{tabular*}
	\caption{Results of multi-dimensional image inpainting by different methods on color images \emph{Butterfly}, \emph{Airplane}, and \emph{Peppers} with SR = $10\%$.}
	\label{fig:rgb}
\end{figure*}
In our method, the factorization rank $r$ is set to $\min(n_1,n_2)$ across all datasets and tasks. The learnable transform $g_{\boldsymbol{\theta}}(\cdot)$ is implemented using three $3\times3$ convolutional layers with LeakyReLU activations. All learnable parameters are optimized using Adam with a learning rate of $8\times10^{-4}$ and a weight decay of $10^{-8}$. The maximum number of iterations is set to $t_{\max}=4500$ for multispectral image recovery and $t_{\max}=1500$ for color image, remote sensing image, and magnetic resonance image recovery. In addition, we employ a pre-trained DINOv3 model\footnote{\url{https://github.com/facebookresearch/dinov3}} to provide meaningful pre-trained knowledge for multi-dimensional image recovery. To quantitatively assess the recovery quality, we adopt two widely used image quality metrics, namely, peak signal-to-noise ratio (PSNR) and structural similarity (SSIM) \cite{wang2004image}, which measure pixel-wise fidelity and structural similarity, respectively. In general, better recovery performance is indicated by higher PSNR and SSIM values.
\subsection{Color Image Recovery}
To illustrate the effectiveness of our method for tensor recovery, we conduct experiments on color images. The tested color images are \emph{Butterfly}, \emph{Airplane}, \emph{Peppers}, and \emph{House}. All color images are of size $256 \times 256 \times 3$. The random sampling strategy is employed to set the sampling rates (SRs) at $2\%$, $5\%$, $10\%$, and $15\%$ for all test data.

Table~\ref{tab:rgbm} reports the quantitative results for color image recovery under random missing observations. PLTD achieves the highest PSNR and SSIM values at most SRs across all test images. Compared with the competitive methods (i.e., MTTD and DELTA), PLTD shows clear advantages across different test images. For instance, PLTD improves the PSNR from 20.591 dB to 25.256 dB on \emph{Butterfly} at $15\%$ SR. Similar gains can also be found on \emph{Airplane}, \emph{Peppers}, and \emph{House}. 

The visual comparison in Fig.~\ref{fig:rgb} further supports the above quantitative results. In particular, TNN, HLRTF, and S2NTNN exhibit limited capability in recovering the overall image structure, whereas MTTD, DTR, and DELTA reconstruct the dominant content but remain limited in preserving fine-grained details. In contrast, PLTD generates visually cleaner and more faithful results, with better color consistency and richer local textures. The zoomed-in regions further demonstrate the superiority of PLTD in accurate boundary delineation and high-frequency detail preservation. These visual results demonstrate that the organic integration of pre-trained knowledge with the learnable low-rank tensor effectively enhances tensor modeling, enabling PLTD to achieve more accurate and visually faithful color image recovery.
\begin{table*}[htbp]
	\caption{Quantitative metrics of the different methods on MSIs under the different sampling conditions.}
	\label{tab:msim}
	\centering
	\footnotesize
	\setlength{\tabcolsep}{1.2pt}
	\renewcommand{\arraystretch}{1.08}
	
	\begin{tabular*}{\linewidth}{@{\extracolsep{\fill}} c c cc cc cc cc cc cc cc}
		\toprule
		\multirow{2}{*}{Data} & \multirow{2}{*}{SR}
		& \multicolumn{2}{c}{TNN}
		& \multicolumn{2}{c}{HLRTF}
		& \multicolumn{2}{c}{S2NTNN}
		& \multicolumn{2}{c}{MTTD}
		& \multicolumn{2}{c}{DTR}
		& \multicolumn{2}{c}{DELTA}
		& \multicolumn{2}{c}{PLTD} \\
		\cmidrule(lr){3-4}\cmidrule(lr){5-6}\cmidrule(lr){7-8}
		\cmidrule(lr){9-10}\cmidrule(lr){11-12}\cmidrule(lr){13-14}
		\cmidrule(lr){15-16}
		& & PSNR & SSIM & PSNR & SSIM & PSNR & SSIM & PSNR & SSIM & PSNR & SSIM & PSNR & SSIM & PSNR & SSIM \\
		\midrule
		
		\multirow{4}{*}{\shortstack{\emph{Balloons}\\(256$\times$256$\times$31)}}
		& $1\%$ & 25.924 & 0.746 & 24.171 & 0.556 & 27.186 & 0.786 & 22.176 & 0.781 & 25.226 & 0.635 & \underline{33.994} & \underline{0.935} & \textbf{35.224} & \textbf{0.967} \\
		& $2\%$ & 29.486 & 0.860 & 30.083 & 0.797 & 31.459 & 0.887 & 28.506 & 0.889 & 34.696 & 0.933 & \underline{37.076} & \underline{0.968} & \textbf{39.145} & \textbf{0.983} \\
		& $5\%$ & 33.421 & 0.923 & 37.003 & 0.944 & 36.971 & 0.958 & 33.145 & 0.945 & 41.255 & 0.985 & \underline{42.073} & \underline{0.988} & \textbf{44.809} & \textbf{0.995} \\
		& $10\%$ & 36.813 & 0.954 & 42.662 & 0.983 & 42.493 & 0.985 & 36.634 & 0.970 & \underline{45.224} & \underline{0.993} & 45.004 & 0.992 & \textbf{48.767} & \textbf{0.997} \\
		\midrule
		
		\multirow{4}{*}{\shortstack{\emph{Feathers}\\(256$\times$256$\times$31)}}
		& $1\%$ & 22.990 & 0.641 & 21.534 & 0.462 & 24.127 & 0.663 & 20.254 & 0.577 & 23.738 & 0.674 & \underline{27.669} & \underline{0.753} & \textbf{29.064} & \textbf{0.881} \\
		& $2\%$ & 25.017 & 0.743 & 24.585 & 0.605 & 27.263 & 0.823 & 23.013 & 0.698 & 27.770 & 0.824 & \underline{29.901} & \underline{0.864} & \textbf{32.006} & \textbf{0.938} \\
		& $5\%$ & 27.941 & 0.832 & 32.261 & 0.893 & 32.097 & 0.921 & 26.169 & 0.819 & \underline{35.515} & \underline{0.957} & 35.102 & 0.951 & \textbf{38.205} & \textbf{0.980} \\
		& $10\%$ & 30.871 & 0.885 & 38.102 & 0.964 & 38.102 & 0.977 & 29.023 & 0.890 & 40.343 & 0.983 & \underline{41.655} & \underline{0.984} & \textbf{43.166} & \textbf{0.991} \\
		\midrule
		
		\multirow{4}{*}{\shortstack{\emph{Flowers}\\(256$\times$256$\times$31)}}
		& $1\%$ & 26.181 & 0.689 & 25.317 & 0.616 & 27.786 & 0.745 & 22.921 & 0.634 & 28.003 & 0.705 & \underline{31.044} & \underline{0.832} & \textbf{32.804} & \textbf{0.908} \\
		& $2\%$ & 29.150 & 0.797 & 28.476 & 0.736 & 30.873 & 0.829 & 26.599 & 0.748 & 31.467 & 0.856 & \underline{34.029} & \underline{0.907} & \textbf{36.083} & \textbf{0.948} \\
		& $5\%$ & 31.990 & 0.869 & 35.232 & 0.913 & 36.094 & 0.932 & 29.848 & 0.833 & 38.246 & 0.945 & \underline{38.701} & \underline{0.953} & \textbf{40.684} & \textbf{0.979} \\
		& $10\%$ & 34.842 & 0.914 & 40.530 & 0.968 & 41.628 & 0.980 & 32.519 & 0.892 & 41.831 & 0.977 & \underline{44.031} & \underline{0.985} & \textbf{44.408} & \textbf{0.990} \\
		\midrule
		
		\multirow{4}{*}{\shortstack{\emph{Toy}\\(256$\times$256$\times$31)}}
		& $1\%$ & 23.558 & 0.723 & 23.798 & 0.658 & 25.225 & 0.751 & 19.940 & 0.586 & 26.306 & 0.745 & \underline{27.791} & \underline{0.800} & \textbf{28.658} & \textbf{0.894} \\
		& $2\%$ & 25.326 & 0.803 & 27.016 & 0.767 & 28.156 & 0.843 & 22.457 & 0.696 & 29.059 & 0.825 & \underline{31.260} & \underline{0.909} & \textbf{32.995} & \textbf{0.951} \\
		& $5\%$ & 28.053 & 0.865 & 33.019 & 0.918 & 33.353 & 0.929 & 25.758 & 0.812 & 36.322 & 0.967 & \underline{37.227} & \underline{0.970} & \textbf{39.617} & \textbf{0.987} \\
		& $10\%$ & 31.094 & 0.905 & 39.125 & 0.974 & 39.082 & 0.978 & 28.585 & 0.886 & 42.154 & 0.990 & \underline{44.861} & \underline{0.994} & \textbf{45.763} & \textbf{0.995} \\
		\bottomrule
	\end{tabular*}
	
	\vspace{2mm}
	The \textbf{best} and \underline{second-best} results are highlighted.
\end{table*}

\begin{figure*}[htbp]
	\footnotesize
	\begin{tabular*}{\linewidth}{CCCCCCCCC@{}}
		Observation & TNN & HLRTF & S2NTNN & MTTD & DTR & DELTA  & PLTD & Ground Truth \\
		
		\includegraphics[width=0.105\linewidth]{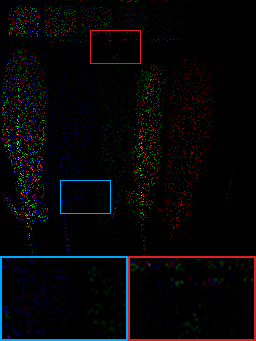} &
		\includegraphics[width=0.105\linewidth]{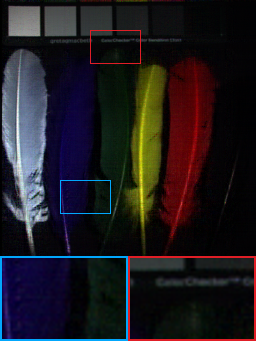} &
		\includegraphics[width=0.105\linewidth]{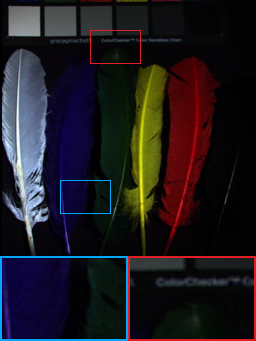} &
		\includegraphics[width=0.105\linewidth]{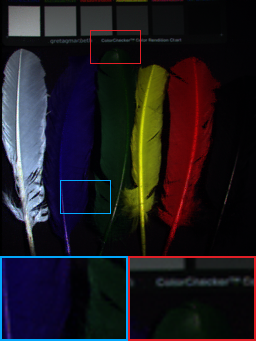} &
		\includegraphics[width=0.105\linewidth]{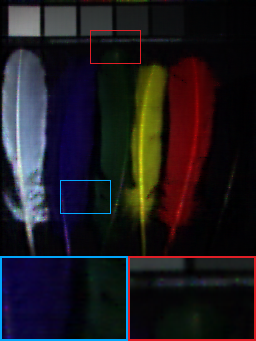} &
		\includegraphics[width=0.105\linewidth]{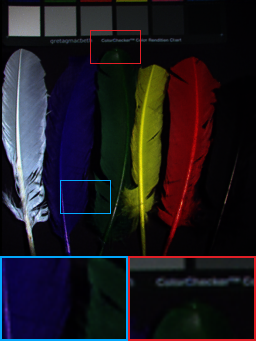} &
		\includegraphics[width=0.105\linewidth]{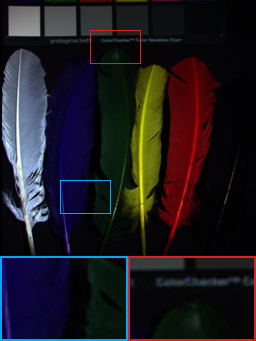} &
		\includegraphics[width=0.105\linewidth]{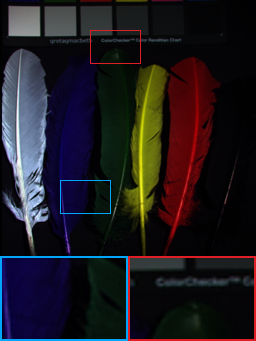} &
		\includegraphics[width=0.105\linewidth]{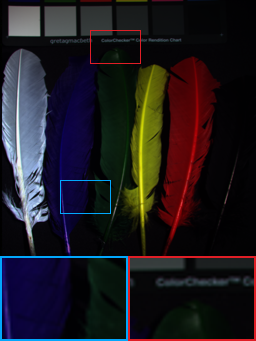} \\
		
		\includegraphics[width=0.105\linewidth]{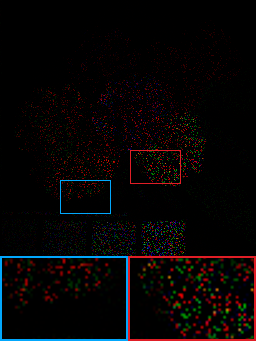} &
		\includegraphics[width=0.105\linewidth]{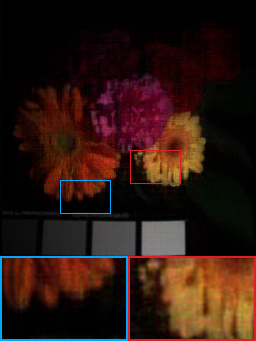} &
		\includegraphics[width=0.105\linewidth]{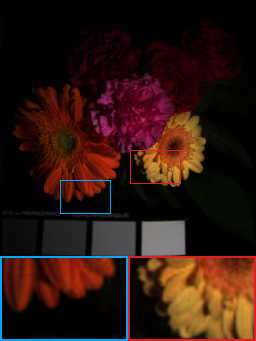} &
		\includegraphics[width=0.105\linewidth]{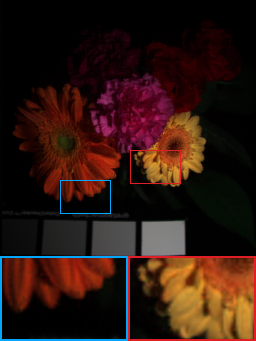} &
		\includegraphics[width=0.105\linewidth]{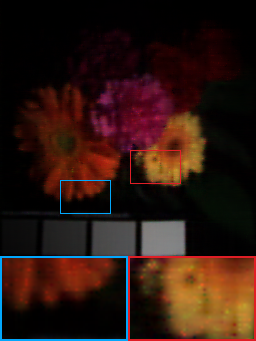} &
		\includegraphics[width=0.105\linewidth]{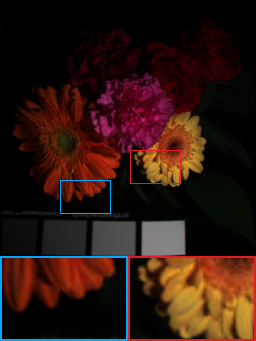} &
		\includegraphics[width=0.105\linewidth]{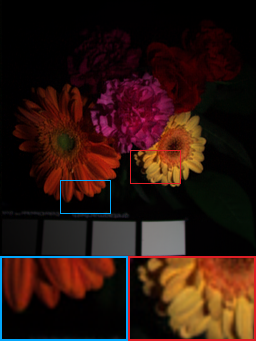} &
		\includegraphics[width=0.105\linewidth]{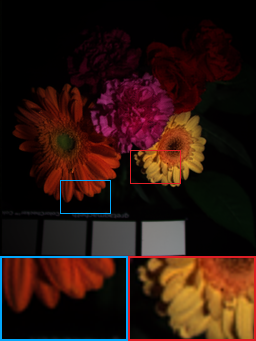} &
		\includegraphics[width=0.105\linewidth]{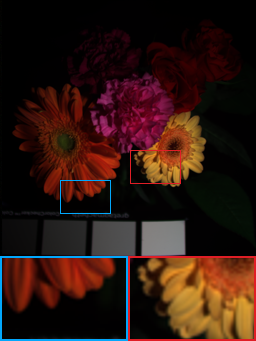} \\
		
	\end{tabular*}
	\caption{Results of multi-dimensional image inpainting by different methods on MSIs \emph{Feathers} and \emph{Flowers} with SR = $10\%$.}
	\label{fig:msi}
\end{figure*}

\subsection{Multispectral Image Recovery}
In this experiment, we investigate the cross-domain generalization ability of PLTD for multispectral image (MSI) recovery. Experiments are conducted on the CAVE dataset, which contains MSIs of size $256 \times 256 \times 31$. Note that the pre-trained knowledge is extracted from DINOv3, which is trained on large-scale RGB images rather than MSIs. Therefore, directly transferring this pre-trained knowledge to MSI recovery introduces a substantial domain gap. 
To enable the pre-trained DINOv3 to process MSIs, we convert each MSI into a pseudo-color image and use it as the input of DINOv3. Specifically, we construct the pseudo-color image using bands 28, 17, and 6 of the CAVE dataset, whose central wavelengths are closest to the red, green, and blue channels, respectively. We evaluate PLTD under two representative degradation settings: random missing and structural missing.

In the case of random missing, we randomly sample MSIs with SRs of $1\%$, $2\%$, $5\%$, and $10\%$. Table~\ref{tab:msim} shows that PLTD consistently outperforms all competing methods across four MSI datasets and four SRs, demonstrating its strong capability to characterize complex multi-dimensional structures.
\begin{figure*}[htbp]
	\centering
	\footnotesize
	\setlength{\tabcolsep}{1.2pt}
	\renewcommand{\arraystretch}{1.02}
	
	\newcommand{\figw}{0.105\linewidth}
	
	\begin{tabular}{@{}ccccccccc@{}}
		
		{\scriptsize PSNR: 20.086} &
		{\scriptsize PSNR: 41.627} &
		{\scriptsize PSNR: 42.732} &
		{\scriptsize PSNR: 42.019} &
		{\scriptsize PSNR: 44.771} &
		{\scriptsize PSNR: 46.618} &
		{\scriptsize PSNR: 47.013} &
		{\scriptsize PSNR: 48.760} &
		{\scriptsize PSNR: inf} \\
		{\scriptsize SSIM: 0.674} &
		{\scriptsize SSIM: 0.987} &
		{\scriptsize SSIM: 0.989} &
		{\scriptsize SSIM: 0.992} &
		{\scriptsize SSIM: 0.994} &
		{\scriptsize SSIM: 0.996} &
		{\scriptsize SSIM: 0.995} &
		{\scriptsize SSIM: 0.998} &
		{\scriptsize SSIM: 1} \\
		
		\includegraphics[width=\figw]{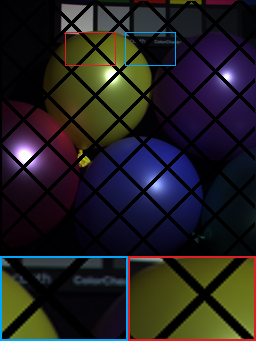} &
		\includegraphics[width=\figw]{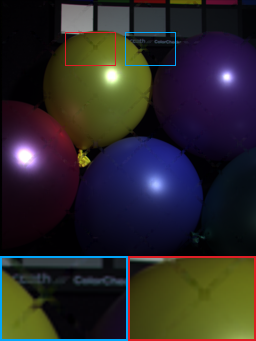} &
		\includegraphics[width=\figw]{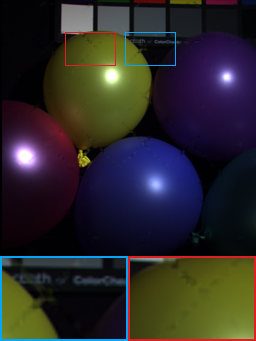} &
		\includegraphics[width=\figw]{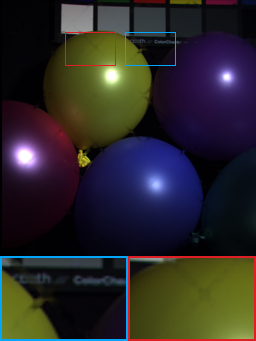} &
		\includegraphics[width=\figw]{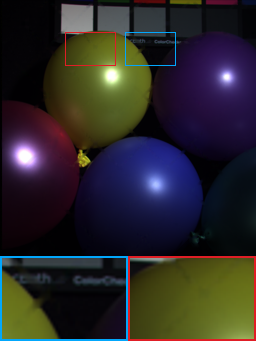} &
		\includegraphics[width=\figw]{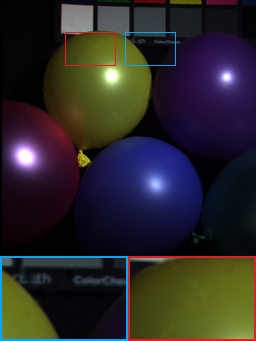} &
		\includegraphics[width=\figw]{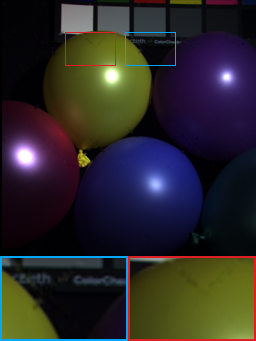} &
		\includegraphics[width=\figw]{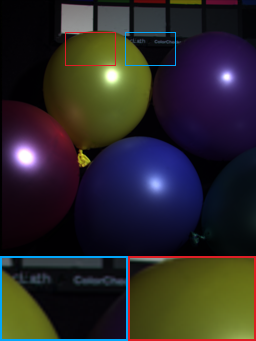} &
		\includegraphics[width=\figw]{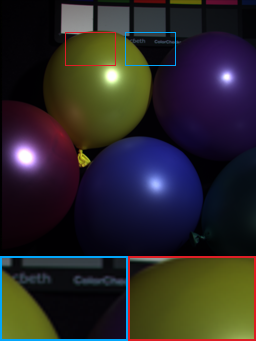} \\
		
		{\scriptsize PSNR: 19.519} &
		{\scriptsize PSNR: 32.441} &
		{\scriptsize PSNR: 33.047} &
		{\scriptsize PSNR: 32.789} &
		{\scriptsize PSNR: 34.052} &
		{\scriptsize PSNR: 33.418} &
		{\scriptsize PSNR: 34.509} &
		{\scriptsize PSNR: 35.276} &
		{\scriptsize PSNR: inf} \\
		{\scriptsize SSIM: 0.823} &
		{\scriptsize SSIM: 0.953} &
		{\scriptsize SSIM: 0.955} &
		{\scriptsize SSIM: 0.956} &
		{\scriptsize SSIM: 0.965} &
		{\scriptsize SSIM: 0.970} &
		{\scriptsize SSIM: 0.961} &
		{\scriptsize SSIM: 0.978} &
		{\scriptsize SSIM: 1} \\
		
		\includegraphics[width=\figw]{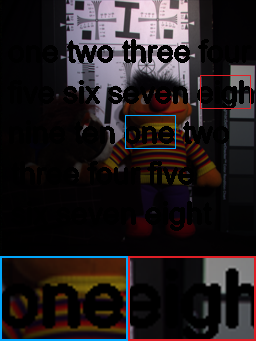} &
		\includegraphics[width=\figw]{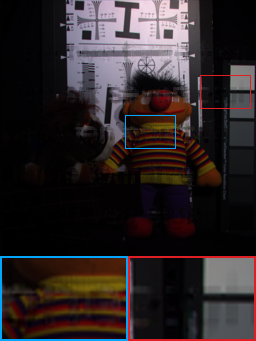} &
		\includegraphics[width=\figw]{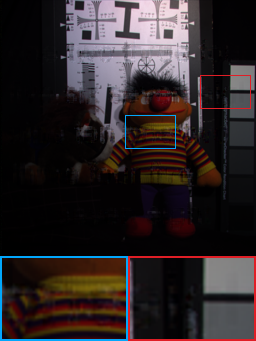} &
		\includegraphics[width=\figw]{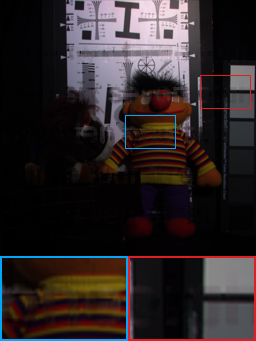} &
		\includegraphics[width=\figw]{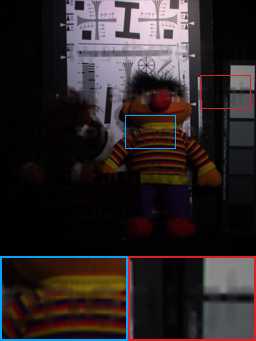} &
		\includegraphics[width=\figw]{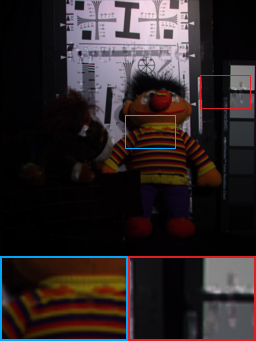} &
		\includegraphics[width=\figw]{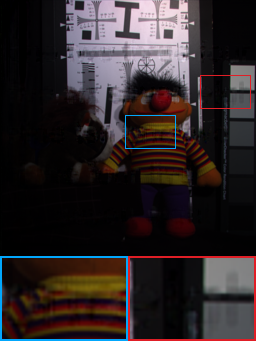} &
		\includegraphics[width=\figw]{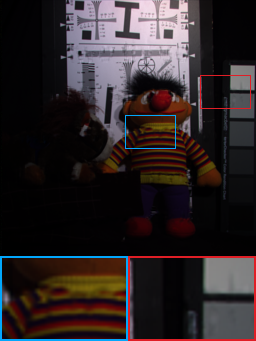} &
		\includegraphics[width=\figw]{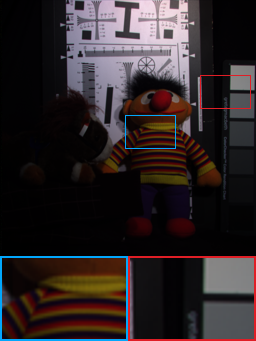} \\
		
		\textbf{Observation} &
		\textbf{TNN} &
		\textbf{HLRTF} &
		\textbf{S2NTNN} &
		\textbf{MTTD} &
		\textbf{DTR} &
		\textbf{DELTA} &
		\textbf{PLTD} &
		\textbf{Ground Truth} \\
	\end{tabular}
	
	\caption{Visual recovery results of tensor structural missing on MSIs \emph{Balloons} and \emph{Toy}, corrupted by black stripes and black text, respectively.}
	\label{fig:structure}
\end{figure*}

\begin{table*}[htbp]
	\caption{Quantitative metrics of the different methods on remote sensing images under the different sampling conditions.}
	\label{tab:rsm}
	\centering
	\footnotesize
	\setlength{\tabcolsep}{1.2pt}
	\renewcommand{\arraystretch}{1.08}
	
	\begin{tabular*}{\linewidth}{@{\extracolsep{\fill}} c c cc cc cc cc cc cc cc}
		\toprule
		\multirow{2}{*}{Data} & \multirow{2}{*}{SR}
		& \multicolumn{2}{c}{TNN}
		& \multicolumn{2}{c}{HLRTF}
		& \multicolumn{2}{c}{S2NTNN}
		& \multicolumn{2}{c}{MTTD}
		& \multicolumn{2}{c}{DTR}
		& \multicolumn{2}{c}{DELTA}
		& \multicolumn{2}{c}{PLTD} \\
		\cmidrule(lr){3-4}\cmidrule(lr){5-6}\cmidrule(lr){7-8}
		\cmidrule(lr){9-10}\cmidrule(lr){11-12}\cmidrule(lr){13-14}
		\cmidrule(lr){15-16}
		& & PSNR & SSIM & PSNR & SSIM & PSNR & SSIM & PSNR & SSIM & PSNR & SSIM & PSNR & SSIM & PSNR & SSIM \\
		\midrule
		
		\multirow{4}{*}{\shortstack{\emph{Sense1}\\(256$\times$256$\times$4)}}
		& $2\%$ & 19.031 & 0.233 & 20.778 & 0.294 & 21.561 & 0.374 & \underline{22.910} & \underline{0.518} & 19.711 & 0.256 & 22.284 & 0.423 & \textbf{23.112} & \textbf{0.538} \\
		& $5\%$ & 21.550 & 0.338 & 22.032 & 0.382 & 22.674 & 0.457 & \underline{24.720} & \underline{0.604} & 22.766 & 0.454 & 24.213 & 0.537 & \textbf{25.053} & \textbf{0.666} \\
		& $10\%$ & 23.142 & 0.456 & 24.162 & 0.555 & 24.402 & 0.581 & \underline{26.532} & \underline{0.703} & 24.554 & 0.605 & 26.489 & 0.696 & \textbf{26.964} & \textbf{0.780} \\
		& $15\%$ & 24.641 & 0.564 & 25.976 & 0.676 & 26.071 & 0.690 & \underline{27.873} & 0.770 & 27.417 & \underline{0.787} & 27.781 & 0.773 & \textbf{28.908} & \textbf{0.840} \\
		\midrule
		
		\multirow{4}{*}{\shortstack{\emph{Sense2}\\(256$\times$256$\times$4)}}
		& $2\%$ & 19.513 & 0.250 & 22.125 & 0.362 & \underline{23.110} & 0.436 & 17.489 & 0.452 & 21.193 & 0.224 & 22.944 & \underline{0.496} & \textbf{24.148} & \textbf{0.588} \\
		& $5\%$ & 22.903 & 0.405 & 23.702 & 0.471 & 24.189 & 0.519 & \underline{25.852} & \underline{0.623} & 23.820 & 0.515 & 25.126 & 0.611 & \textbf{25.929} & \textbf{0.675} \\
		& $10\%$ & 24.451 & 0.524 & 25.500 & 0.607 & 25.664 & 0.627 & \underline{27.543} & \underline{0.717} & 26.999 & 0.697 & 27.055 & 0.705 & \textbf{28.071} & \textbf{0.776} \\
		& $15\%$ & 25.854 & 0.615 & 27.170 & 0.717 & 27.358 & 0.731 & \underline{28.949} & \underline{0.784} & 28.429 & 0.765 & 28.531 & 0.770 & \textbf{29.662} & \textbf{0.839} \\
		\midrule
		
		\multirow{4}{*}{\shortstack{\emph{Sense3}\\(256$\times$256$\times$4)}}
		& $2\%$ & 18.273 & 0.187 & 20.818 & 0.235 & 21.503 & 0.274 & 16.626 & \underline{0.302} & 19.697 & 0.192 & \underline{22.455} & \underline{0.302} & \textbf{22.964} & \textbf{0.426} \\
		& $5\%$ & 21.485 & 0.283 & 22.299 & 0.368 & 22.851 & 0.407 & \textbf{24.592} & \underline{0.468} & 21.674 & 0.307 & 24.184 & 0.439 & \underline{24.466} & \textbf{0.554} \\
		& $10\%$ & 22.959 & 0.399 & 24.092 & 0.530 & 24.304 & 0.532 & \underline{25.583} & 0.559 & 24.733 & 0.571 & 25.355 & \underline{0.592} & \textbf{26.160} & \textbf{0.686} \\
		& $15\%$ & 24.420 & 0.518 & 25.748 & 0.652 & 25.863 & 0.658 & \underline{27.527} & 0.701 & 26.536 & 0.675 & 26.947 & \underline{0.702} & \textbf{27.846} & \textbf{0.766} \\
		\midrule
		
		\multirow{4}{*}{\shortstack{\emph{Sense4}\\(256$\times$256$\times$4)}}
		& $2\%$ & 17.057 & 0.157 & 18.270 & 0.180 & 18.847 & 0.222 & \underline{20.992} & \underline{0.332} & 18.708 & 0.192 & 20.556 & 0.306 & \textbf{21.742} & \textbf{0.478} \\
		& $5\%$ & 19.433 & 0.233 & 20.183 & 0.323 & 20.213 & 0.336 & \underline{22.877} & \underline{0.469} & 20.895 & 0.393 & 22.601 & 0.467 & \textbf{23.423} & \textbf{0.615} \\
		& $10\%$ & 20.949 & 0.364 & 22.201 & 0.509 & 22.415 & 0.525 & \underline{24.577} & 0.616 & 23.598 & \underline{0.632} & 24.033 & 0.580 & \textbf{25.236} & \textbf{0.755} \\
		& $15\%$ & 22.399 & 0.483 & 23.963 & 0.637 & 23.901 & 0.638 & 26.090 & 0.721 & 25.625 & 0.746 & \underline{26.583} & \underline{0.766} & \textbf{27.043} & \textbf{0.823} \\
		\bottomrule
	\end{tabular*}
	
	\vspace{2mm}
	The \textbf{best} and \underline{second-best} results are highlighted.
\end{table*}
Figure~\ref{fig:msi} further provides a visual comparison at a sampling rate of $10\%$. It can be observed that the competing methods suffer from different levels of detail degradation, including blurred edges and indistinct textural details. In contrast, PLTD recovers more faithful image structure and sharper local textures. These visual results are consistent with the quantitative improvements in Table~\ref{tab:msim}, further demonstrating that the organic integration of pre-trained knowledge with the learnable low-rank tensor enables PLTD to achieve strong cross-domain generalization and accurate recovery under different SRs.

In the case of structural missing, we consider two representative structural missing patterns: black stripes and black text types.
Unlike random missing, structural missing severely damages local structures in spatially contiguous regions, substantially reducing the available local information and making recovery more dependent on global structural information. 
Fig.~\ref{fig:structure} presents the visual recovery results under different structural missing patterns. 
As shown in the enlarged regions, the competing methods exhibit limited capability in faithfully reconstructing the missing local content. TNN, HLRTF, and S2NTNN produce blurred or structurally distorted patches, whereas MTTD, DTR, and DELTA recover the dominant appearance but still fail to accurately preserve fine boundaries and local textures. In contrast, PLTD reconstructs the missing regions with more coherent structures and clearer local details. These results demonstrate that pre-trained knowledge provides reliable global guidance, while its organic integration with the learnable low-rank tensor enables PLTD to effectively exploit such guidance and faithfully preserve the underlying structural information even under severe structural missing.
\subsection{Remote Sensing Image Recovery}
In this experiment, we further evaluate the generalization capability of PLTD on more challenging real-world remote sensing images (i.e., high-resolution NAIP CNIR images). The images have a spatial resolution of $0.6$ m and consist of four spectral bands, namely red, green, blue, and near-infrared (NIR).
As shown in Table~\ref{tab:rsm}, PLTD consistently achieves higher PSNR and SSIM values under different sampling rates, demonstrating its strong robustness and recovery accuracy. 
More importantly, the improvement in SSIM indicates that the proposed method can better preserve structural consistency in remote sensing images.
The visual comparisons in Fig.~\ref{fig:rs} further verify the effectiveness of PLTD. 
Traditional tensor decomposition methods tend to produce over-smoothed results and lose fine spatial details. 
In contrast, PLTD recovers clearer textures, sharper field boundaries, and more faithful structures. 
The zoomed-in regions show that our method better preserves local texture details and structural edges.
These results demonstrate that the organic integration of pre-trained knowledge with the learnable low-rank tensor enables PLTD to effectively model complex structures and generalize well to real-world remote sensing images.

\begin{figure*}[htbp]
	\centering
	\footnotesize
	\begin{tabular*}{\linewidth}{CCCCCCCCC@{}}
		Observation & TNN & HLRTF & S2NTNN & MTTD & DTR & DELTA  & PLTD & Ground Truth \\
		
		\includegraphics[width=0.105\linewidth]{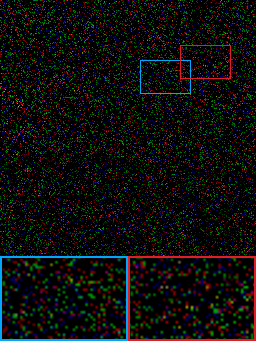} &
		\includegraphics[width=0.105\linewidth]{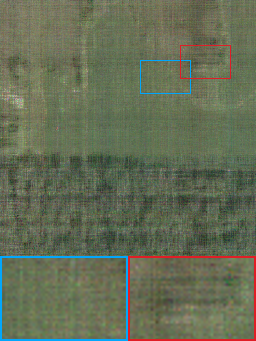} &
		\includegraphics[width=0.105\linewidth]{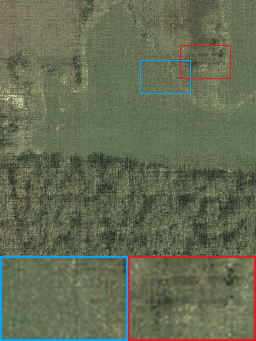} &
		\includegraphics[width=0.105\linewidth]{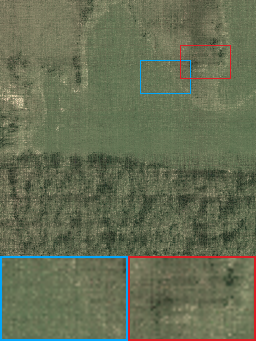} &
		\includegraphics[width=0.105\linewidth]{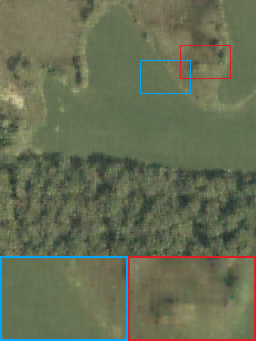} &
		\includegraphics[width=0.105\linewidth]{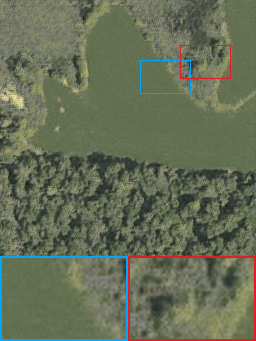} &
		\includegraphics[width=0.105\linewidth]{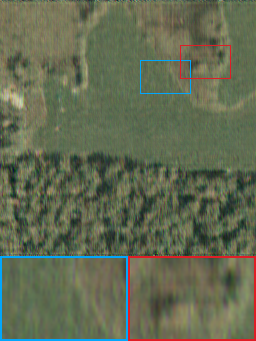} &
		\includegraphics[width=0.105\linewidth]{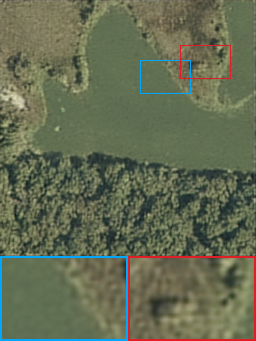} &
		\includegraphics[width=0.105\linewidth]{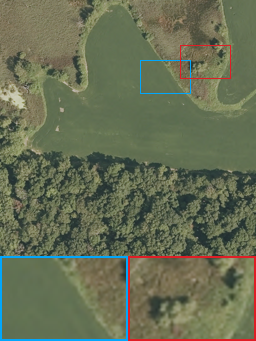} \\
		
		\includegraphics[width=0.105\linewidth]{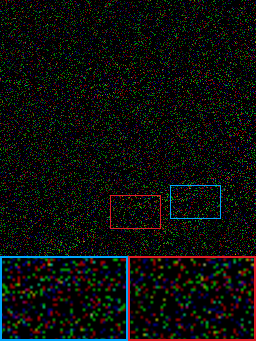} &
		\includegraphics[width=0.105\linewidth]{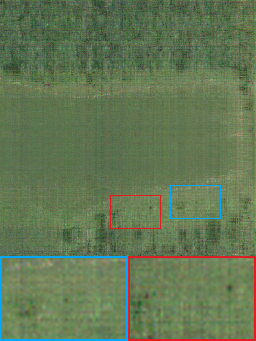} &
		\includegraphics[width=0.105\linewidth]{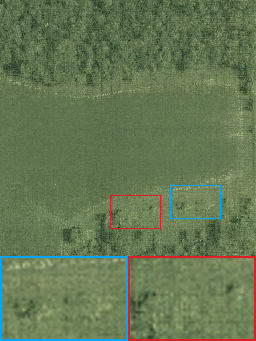} &
		\includegraphics[width=0.105\linewidth]{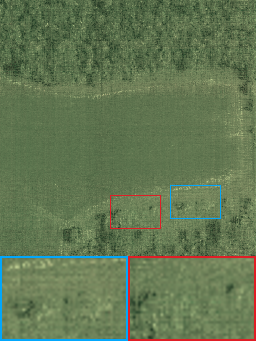} &
		\includegraphics[width=0.105\linewidth]{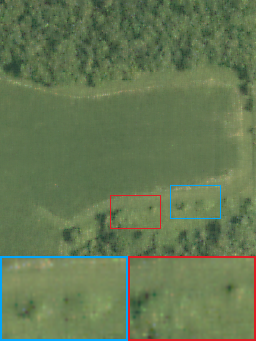} &
		\includegraphics[width=0.105\linewidth]{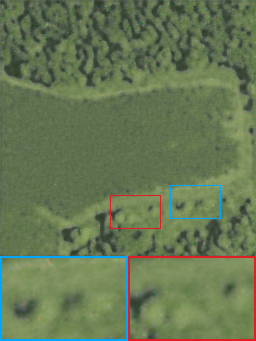} &
		\includegraphics[width=0.105\linewidth]{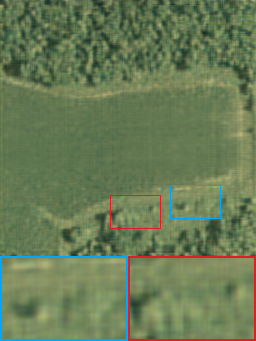} &
		\includegraphics[width=0.105\linewidth]{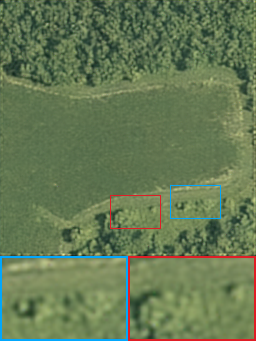} &
		\includegraphics[width=0.105\linewidth]{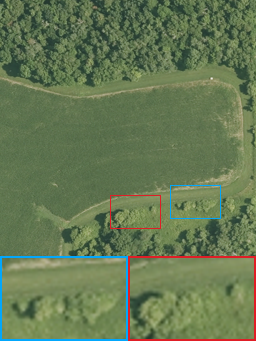} \\
	\end{tabular*}
	\caption{Results of multi-dimensional image inpainting by different methods on remote sensing images \emph{Sense1} and \emph{Sense2} with SR = $10\%$.}
	\label{fig:rs}
\end{figure*}

\subsection{Magnetic Resonance Imaging Recovery}
\begin{figure*}[htbp]
	\centering
	\footnotesize
	\begin{tabular*}{\linewidth}{CCCCCCCCC@{}}
		Observation & TNN & HLRTF & S2NTNN & MTTD & DTR & DELTA & PLTD & Ground Truth \\
		
		\includegraphics[width=0.105\linewidth]{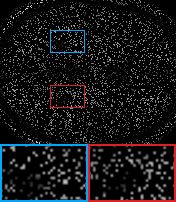} &
		\includegraphics[width=0.105\linewidth]{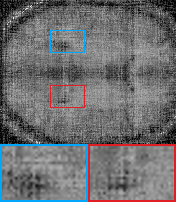} &
		\includegraphics[width=0.105\linewidth]{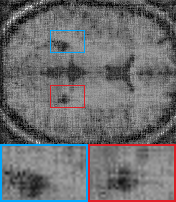} &
		\includegraphics[width=0.105\linewidth]{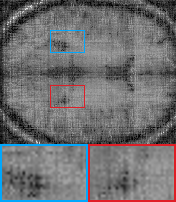} &
		\includegraphics[width=0.105\linewidth]{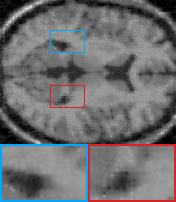} &
		\includegraphics[width=0.105\linewidth]{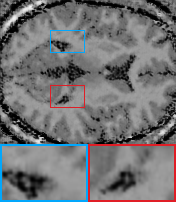} &
		\includegraphics[width=0.105\linewidth]{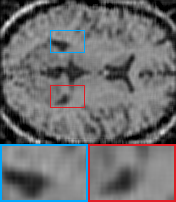} &
		\includegraphics[width=0.105\linewidth]{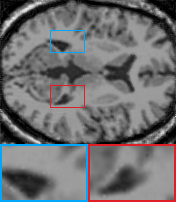} &
		\includegraphics[width=0.105\linewidth]{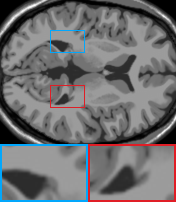} \\
		
		\includegraphics[width=0.105\linewidth]{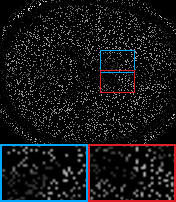} &
		\includegraphics[width=0.105\linewidth]{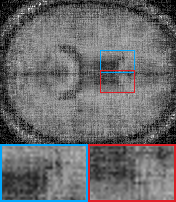} &
		\includegraphics[width=0.105\linewidth]{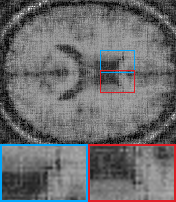} &
		\includegraphics[width=0.105\linewidth]{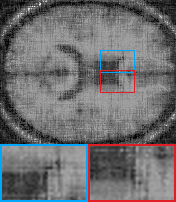} &
		\includegraphics[width=0.105\linewidth]{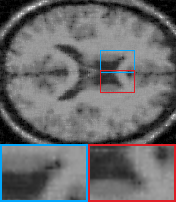} &
		\includegraphics[width=0.105\linewidth]{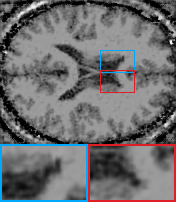} &
		\includegraphics[width=0.105\linewidth]{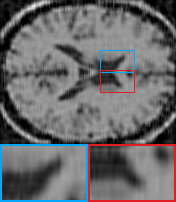} &
		\includegraphics[width=0.105\linewidth]{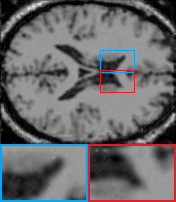} &
		\includegraphics[width=0.105\linewidth]{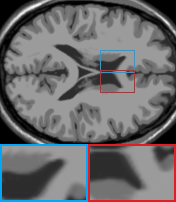} \\
	\end{tabular*}
	
	\caption{Results of multi-dimensional image inpainting by different methods on brain MRI slices at different locations with SR = $15\%$.}
	\label{fig:mri}
\end{figure*}
\begin{table}[htbp]
	\centering
	\caption{Quantitative comparison of different methods for MRI completion on data of size $144\times176\times3$. The \textbf{best} results are in bold and the second-best results are underlined.}
	\label{tab:mri_completion}
	\setlength{\tabcolsep}{5.5pt}
	\renewcommand{\arraystretch}{1.08}
	\begin{tabular}{lcccccc}
		\toprule
		SR 
		& \multicolumn{2}{c}{0.10}
		& \multicolumn{2}{c}{0.15}
		& \multicolumn{2}{c}{0.20} \\
		\cmidrule(lr){2-3}
		\cmidrule(lr){4-5}
		\cmidrule(lr){6-7}
		Method
		& PSNR & SSIM
		& PSNR & SSIM
		& PSNR & SSIM \\
		\midrule
		
		TNN
		& 14.660 & 0.190
		& 15.800 & 0.269
		& 16.917 & 0.349 \\
		
		HLRTF
		& 15.539 & 0.326
		& 16.787 & 0.466
		& 17.853 & 0.576 \\
		
		S2NTNN
		& 15.675 & 0.334
		& 17.177 & 0.478
		& 18.404 & 0.596 \\
		
		MTTD
		& \underline{19.511} & \underline{0.607}
		& \underline{21.194} & 0.698
		& \underline{22.638} & 0.762 \\
		
		DTR
		& 15.729 & 0.496
		& 17.037 & 0.640
		& 18.077 & \underline{0.804} \\
		
		DELTA
		& 19.086 & 0.594
		& 20.685 & \underline{0.738}
		& 21.509 & 0.766 \\
		
		PLTD
		& \textbf{20.608} & \textbf{0.809}
		& \textbf{22.087} & \textbf{0.868}
		& \textbf{23.217} & \textbf{0.921} \\
		
		\bottomrule
	\end{tabular}
\end{table}
In this experiment, we investigate the generalization capability of PLTD on magnetic resonance images (MRI). Compared with the color images, MRI exhibits substantially different visual characteristics and imaging mechanisms, introducing a more challenging domain gap. We conduct recovery experiments under random missing observations with SRs of $10\%$, $15\%$, and $20\%$.
As shown in Table~\ref{tab:mri_completion}, PLTD consistently achieves the highest PSNR and SSIM values under all sampling rates, demonstrating its strong recovery capability on MRI data. In particular, at an SR of 10\%, PLTD achieves a PSNR of 20.608 dB and an SSIM of 0.809, outperforming the second-best results by 1.097 dB and 0.202, respectively. Similar advantages can also be observed at higher sampling rates. 
The visual comparisons in Fig.~\ref{fig:mri} further verify the effectiveness of PLTD. The competing methods suffer from different degrees of structural distortion and blurred boundaries. In contrast, PLTD recovers more coherent structures and clearer tissue boundaries, producing results that are visually closer to the ground truth. These results demonstrate that the organic integration of pre-trained common structure with the learnable low-rank tensor enables PLTD to effectively generalize to medical images with a substantial domain gap. 
\section{Discussion}\label{sec6}
In this section, we provide further discussions on several perspectives of the proposed PLTD framework.
We first discuss the benefits of introducing pre-trained common structure into tensor decomposition. Then, we investigate the complementary roles of the pre-trained latent tensor and the learnable low-rank latent tensor, revealing how they jointly contribute to effective tensor modeling. Furthermore, we investigate the influence of different low-rank structures on modeling instance-specific structure in PLTD. Finally, we examine the applicability of PLTD to high-level vision tasks and analyze its convergence behavior to further demonstrate the benefits of incorporating pre-trained common structure.

\subsection{Benefits of Pre-Trained Common Structure}
To the best of our knowledge, this work is the first to introduce a pre-trained large vision model into tensor decomposition. Traditional tensor decomposition methods mainly characterize the instance-specific structure of each target tensor independently. However, tensor decompositions neglect the common structure across different images, leading to limited modeling capability, high computational cost, and a large number of learnable parameters. By incorporating the common structure learned from large vision models, PLTD extends tensor modeling beyond instance-specific structure to jointly characterize common and instance-specific structure.
Compared with classical instance-specific tensor modeling, PLTD leverages informative common structure while learning only a lightweight low-rank tensor for instance-specific modeling. As shown in Table~\ref{tab:parameter_efficiency}, PLTD achieves more favorable recovery performance with significantly fewer learnable parameters and shorter running time. 

This comparison further provides new insights into the design of next-generation tensor decomposition methods. Rather than relying solely on instance-specific modeling, future tensor decomposition methods can leverage common structure as a foundation while integrating low-rank tensor decomposition to characterize the instance-specific structure of the target tensor. This shift from instance-specific modeling toward common structure modeling opens a new direction for developing more expressive and efficient tensor decomposition frameworks.
\begin{table}[htbp]
	\centering
	\caption{Comparison of recovery performance, model parameters, and running time on the MSI \emph{Balloons} dataset.}
	\label{tab:parameter_efficiency}
	\setlength{\tabcolsep}{3pt}
	\renewcommand{\arraystretch}{1.18}
	\definecolor{lightorange}{RGB}{245,226,216}
	
	\resizebox{\linewidth}{!}{
		\begin{tabular}{lccccc}
			\toprule
			\textbf{Method} & \textbf{Transform Arch}  & \textbf{PSNR$\uparrow$} & \textbf{SSIM$\uparrow$} & \textbf{Params} & \textbf{Time}\\
			\midrule
			HLRTF  & FCN              & 30.083 dB & 0.797 & 3.165 M  & 10.235 s  \\
			S2NTNN & FCN             & 31.459 dB & 0.887 & 2.170 M  & 85.832 s  \\			
			\rowcolor{lightorange}
			PLTD & FCN             & \textbf{32.458 dB} & \textbf{0.913} & \textbf{0.171 M} & \textbf{4.221 s}  \\
			\midrule
			DTR  & CNN & 34.696 dB & 0.933 & 4.080 M & 92.352 s  \\
			DELTA  & CNN \& Transformer & 37.076 dB & 0.968 & 10.596 M & 86.396 s  \\
			\rowcolor{lightorange}
			PLTD & CNN             & \textbf{39.145 dB} & \textbf{0.983} & \textbf{0.988 M}  & \textbf{19.170 s}   \\
			\bottomrule
		\end{tabular}
	}
\end{table}

\subsection{Synergy between Pre-Trained Latent Tensor and Learnable Low-Rank Latent Tensor}\label{subsec2}
To further understand the latent tensor in PLTD, we analyze the synergy between the fixed pre-trained latent tensor and the learnable low-rank latent tensor. As shown in Table~\ref{tab:comparison}, using only the low-rank latent tensor (i.e., $\mathcal{X} = g_{\boldsymbol\theta} ((\mathbf A\mathbf B)\circ\mathbf c)$) leads to inferior recovery performance, especially at low sampling rates. Using only the pre-trained latent tensor (i.e., $\mathcal{X} = g_{\boldsymbol\theta} (\mathcal{S})$) also achieves limited recovery performance, which improves slightly as the sampling rate increases. By integrating the two components, PLTD consistently achieves the best performance, demonstrating their complementary roles in tensor recovery.

Fig.~\ref{fig:component} displays that the pre-trained latent tensor provides a reliable common structure, while the learnable low-rank latent tensor complements this foundation by characterizing the instance-specific structure of the target tensor. These results further show that the pre-trained latent tensor and the learnable low-rank latent tensor play complementary roles in PLTD. 
\begin{table}[htbp]
	\centering
	\footnotesize
	\caption{Comparison of the pre-trained frozen latent tensor and learnable low-rank latent tensor under different sampling rates on the \emph{Balloons} dataset.}
	\label{tab:comparison}
	\setlength{\tabcolsep}{6pt}
	\renewcommand{\arraystretch}{1.15}
	\begin{tabular}{c cc cc cc}
		\toprule
		\multirow{2}{*}{SR} 
		& \multicolumn{2}{c}{Low-Rank Only} 
		& \multicolumn{2}{c}{Pre-Trained Only} 
		& \multicolumn{2}{c}{PLTD} \\
		\cmidrule(lr){2-3} \cmidrule(lr){4-5} \cmidrule(lr){6-7}
		& PSNR & SSIM & PSNR & SSIM & PSNR & SSIM \\
		\midrule
		1\%  & 22.517 & 0.599 & 30.708 & 0.929 & \textbf{35.224} & \textbf{0.967} \\
		2\%  & 27.773 & 0.692 & 34.227 & 0.967 & \textbf{39.145} & \textbf{0.983} \\
		5\%  & 39.575 & 0.969 & 37.852 & 0.979 &\textbf{44.809} & \textbf{0.995} \\
		10\% & 43.538 & 0.987 & 39.518 & 0.984 &\textbf{48.767} & \textbf{0.997} \\
		\bottomrule
	\end{tabular}
\end{table}
\begin{figure}[htbp]
	\centering
	\setlength{\tabcolsep}{2pt}
	\begin{tabular}{@{}cccc@{}}
		{\scriptsize Low-Rank Only} & {\scriptsize Pre-Trained Only} & {\scriptsize PLTD} & {\scriptsize Ground Truth} \\
		\includegraphics[width=2cm]{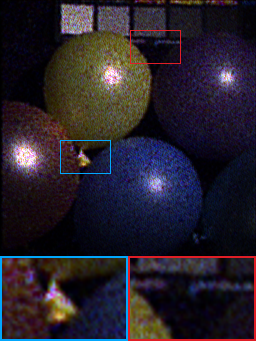} &
		\includegraphics[width=2cm]{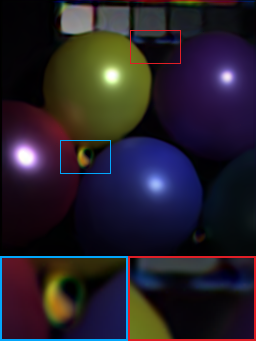} &
		\includegraphics[width=2cm]{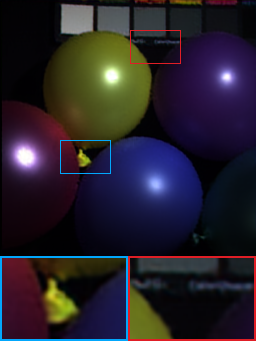} &
		\includegraphics[width=2cm]{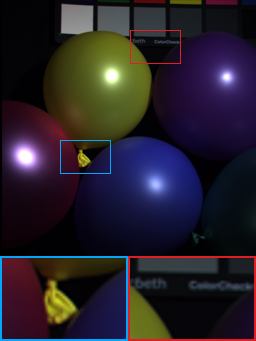} \\
	\end{tabular}
	\caption{Visual comparison of low-rank only, pre-trained only, and PLTD. The pre-trained latent tensor provides the principal structural information, while the learnable low-rank latent tensor complements it with local details. The two components are complementary and jointly indispensable for accurately characterizing the latent tensor.}
	\label{fig:component}
\end{figure}
\begin{figure*}[t]
	\centering
	\setlength{\tabcolsep}{4pt}
	\begin{tabular}{c c c}
		\includegraphics[width=0.31\textwidth]{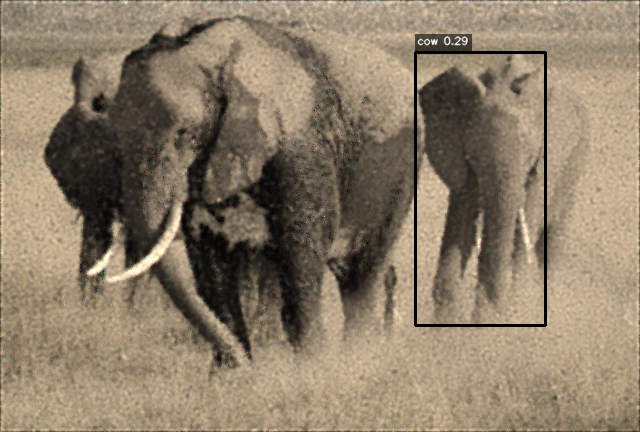} &
		\includegraphics[width=0.31\textwidth]{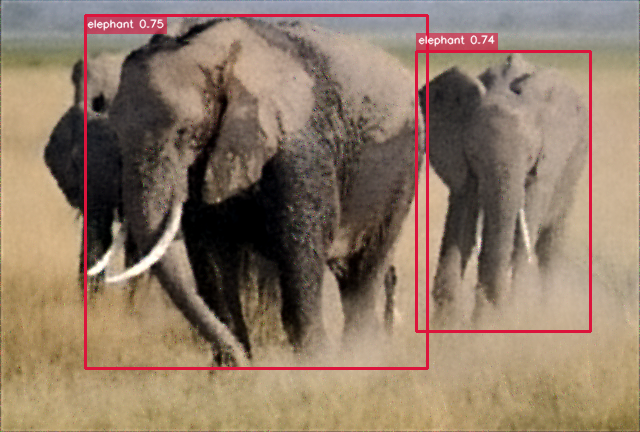} &
		\includegraphics[width=0.31\textwidth]{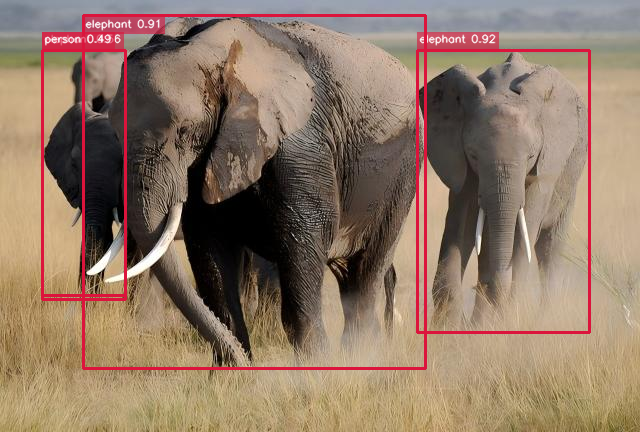} \\
		{\scriptsize w/o Pre-Trained Semantic Knowledge} & {\scriptsize PLTD} & {\scriptsize Ground Truth} \\
	\end{tabular}
	\caption{Object detection results on recovery images using YOLOv11. Benefiting from pre-trained semantic knowledge, the proposed method achieves higher detection confidence.}
	\label{fig:object_detect}
\end{figure*}

\subsection{Other Low-Rank Structures in PLTD} \label{dis: extension}
To investigate the influence of different low-rank structures on instance-specific characterization, we consider different low-rank tensor decompositions in PLTD. Specifically, we consider three classical low-rank structures, including CP decomposition, Tucker decomposition, and slice-wise matrix factorization (MF).

As shown in Table~\ref{tab:low_rank_structure}, different low-rank structures achieve comparable recovery performance, demonstrating the flexibility of PLTD in accommodating various low-rank representations. Among them, our low-rank structure provides a favorable trade-off among recovery fidelity, the number of learnable parameters, and computational efficiency. Specifically, it achieves the highest PSNR and SSIM values while requiring fewer learnable parameters and a shorter running time. Therefore, we adopt this low-rank structure in PLTD for efficient instance-specific characterization of the target tensor.

\begin{table}[htbp]
	\centering
	\caption{Comparison of different low-rank structures.}
	\label{tab:low_rank_structure}
	\setlength{\tabcolsep}{3.7pt}
	\renewcommand{\arraystretch}{1.18}
	\begin{tabular}{ccccc}
		\toprule
		\textbf{Low-Rank Structure} 
		& \textbf{PSNR}$\uparrow$
		& \textbf{SSIM}$\uparrow$
		& \textbf{Params}
		& \textbf{Time} \\
		\midrule
		CP
		& 38.895 dB
		& 0.982 
		& 0.991 M 
		& 23.916 s \\
		Tucker
		& 39.138 dB
		& 0.982 
		& 0.995 M 
		& 23.712 s \\
		
		MF
		& 39.050 dB
		& 0.982 
		& 3.216 M 
		& 20.241 s \\
		
		Our
		& \textbf{39.145 dB}
		& \textbf{0.983 }
		& \textbf{0.988 M }
		& \textbf{19.170 s} \\
		\bottomrule
	\end{tabular}
\end{table}

\subsection{High-Level Vision Task}
Benefiting from the semantic knowledge inherited from the pre-trained DINOv3 model, PLTD not only improves tensor recovery performance but also preserves high-level visual semantics in the reconstructed results. Such semantic-aware recovery enables better compatibility with downstream high-level vision tasks. Accordingly, we perform object detection using YOLOv11~\cite{khanam2024yolov11} on the reconstructed results, where the test image is sampled from the COCO dataset\footnote{\url{https://cocodataset.org}}.
As shown in Fig.~\ref{fig:object_detect}, by leveraging pre-trained semantic knowledge, PLTD enables YOLOv11 to achieve more reliable detection with higher confidence scores.

In contrast, without the pre-trained semantic knowledge, the recovery image may suffer from semantic errors, which lead to incorrect or missing detections. For example, the elephant on the right is mistakenly recognized as a cow, while the elephant on the far left is not successfully detected. These results indicate that pre-trained semantic knowledge facilitates the recovery of semantically discriminative structures, enabling the recovery image to maintain high pixel-level fidelity while better preserving compatibility with high-level vision tasks.

\subsection{Convergence Behavior Analysis}
In this subsection, we further analyze the convergence behavior of the proposed solving algorithm. Specifically, we analyze the loss function values and recovery results of the proposed PLTD-based recovery model with and without the pre-trained latent tensor during optimization.

As shown in Fig.~\ref{fig:loss}, PLTD exhibits a much faster decrease in the loss function value and reaches a low loss function value substantially earlier than the model without the pre-trained latent tensor. This result indicates that the pre-trained latent tensor provides a favorable foundation for tensor recovery, allowing the model to reach a low loss function value with substantially fewer iterations. 

The convergence effectiveness of PLTD is further reflected in the early-stage recovery results shown in Fig.~\ref{fig:Early reconstruction quality}. After only 200 iterations, PLTD already produces visually faithful results with clear structures and details, whereas the model without the pre-trained latent tensor still exhibits noticeable distortions and color inconsistencies. In addition, the spectrum plots show that PLTD recovers the dominant frequency components at an earlier stage, further demonstrating its superior early-stage recovery capability.
\begin{figure}[htbp]
	\centering
	\includegraphics[width=6.5cm]{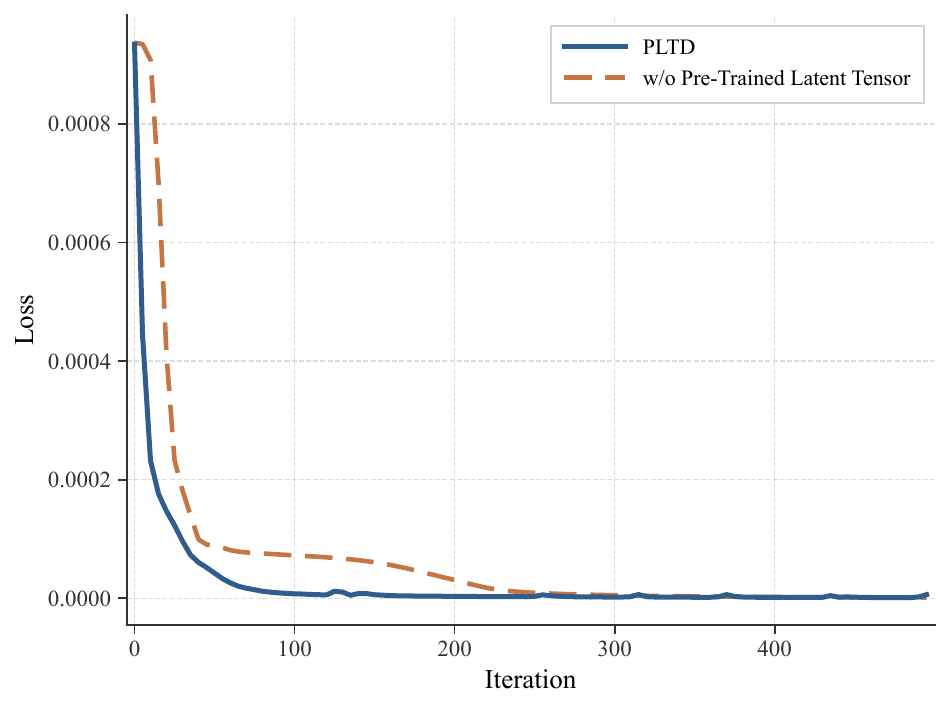}
	\caption{Optimization loss curves with and without the pre-trained latent tensor during the recovery process.}
	\label{fig:loss}
\end{figure}	

\begin{figure}[htbp]
	\centering
	\setlength{\tabcolsep}{2pt} 
	\begin{tabular}{@{}cccc@{}}
		{\scriptsize Observed} & {\scriptsize w/o Pre-Trained} & {\scriptsize PLTD} & {\scriptsize Ground Truth} \\
		\includegraphics[width=2cm]{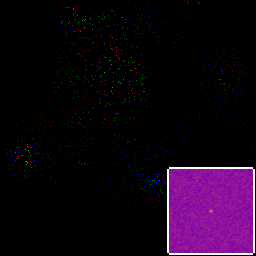} &
		\includegraphics[width=2cm]{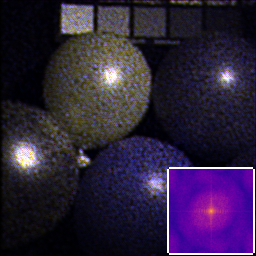} &
		\includegraphics[width=2cm]{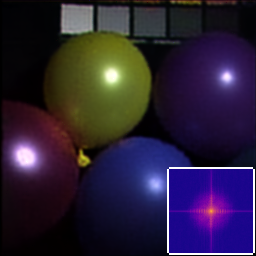} &
		\includegraphics[width=2cm]{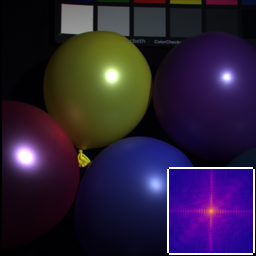} \\
		
		{\scriptsize PSNR: 11.933} & {\scriptsize PSNR: 24.864} & {\scriptsize PSNR: 35.656} & {\scriptsize PSNR: inf} \\
		[-2.59pt]
		{\scriptsize SSIM: 0.342} & {\scriptsize SSIM: 0.611} & {\scriptsize SSIM: 0.957} & {\scriptsize SSIM: 1} \\
	\end{tabular}
	\caption{Early-stage recovery results after 200 iterations. The results show that PLTD already achieves favorable recovery quality at an early stage, with clear spatial structures and faithful spectral information, demonstrating the effectiveness of the proposed PLTD.}
	\label{fig:Early reconstruction quality}
\end{figure}

\section{Conclusion}\label{sec7}
In this work, we introduced pre-training into tensor decomposition and proposed a pre-trained low-rank tensor decomposition (PLTD) framework. Existing tensor decomposition methods characterize the instance-specific structure of each target tensor from scratch, neglecting the common structure across different images. PLTD incorporates common structure inherited from a pre-trained large vision model and employs a lightweight low-rank tensor to characterize instance-specific structure. Extensive experiments on color images, multispectral images, real-world remote sensing images, and magnetic resonance images demonstrate that PLTD consistently achieves superior recovery performance with significantly fewer learnable parameters and a smaller carbon footprint.
Looking forward, this work points to a promising direction for next-generation tensor decomposition, where pre-trained large vision models provide transferable common structures for tensor modeling. By integrating such common structures with compact low-rank modeling, future tensor decomposition methods may achieve more expressive and efficient characterization of multi-dimensional data. 
\bibliographystyle{IEEEtran}
\bibliography{reference}
\newpage

\allowdisplaybreaks

\appendices
\section{Proof of Theorem}
\label{appendix}

\setcounter{definition}{0}
\setcounter{lemma}{0}
\setcounter{theorem}{0}
This appendix establishes the theoretical recovery error bounds for the proposed model. We first introduce the concepts of $\epsilon$-nets and covering numbers, followed by a preliminary lemma on Euclidean balls to support the subsequent analysis.

\begin{definition}[$\epsilon$-Net \cite{Vershynin_2018}] \label{def:06} 
	Let $(\mathfrak{T}, d)$ be a metric space, where $d:\mathfrak{T}\times\mathfrak{T}\to\mathbb{R}_+$ denotes a metric.
	Consider a subset $\mathfrak{L} \subset \mathfrak{T}$ and let $\epsilon>0$.
	A subset $\bar{\mathfrak{L}} \subseteq \mathfrak{L}$ is called an $\epsilon$-net of $\mathfrak{L}$ if every point in $\mathfrak{L}$ is within distance $\epsilon$ of some point of $\bar{\mathfrak{L}}$, i.e.,
	\[
	\forall \mathcal{X} \in \mathfrak{L},\ \exists\ \mathcal{X}_0 \in \bar{\mathfrak{L}}:\ 
	d(\mathcal{X}, \mathcal{X}_0)\le \epsilon .
	\]
\end{definition}

\begin{definition}[$\epsilon$-Covering Number \cite{Vershynin_2018}] \label{def:07}
	The smallest possible cardinality of an $\epsilon$-net of $\mathfrak{L}$ is called the covering number of $\mathfrak{L}$ and is denoted by $\mathcal{N}(\mathfrak{L}, d, \epsilon)$.
	In this paper, we mainly use the Frobenius-norm metric $d(\mathcal{X},\mathcal{X}')=\|\mathcal{X}-\mathcal{X}'\|_F$.
\end{definition}

\begin{lemma}[Covering numbers of the Euclidean ball \cite{Vershynin_2018}]
	The covering number of an \( n \)-dimensional Euclidean ball \( B_2^n(\rho) \) with radius \( \rho \) satisfies the following for any $0<\epsilon\leq\rho$:
	\begin{equation}
		\mathcal{N}\left(B_2^n(\rho), \|\cdot\|_2, \epsilon\right) \leq\left(\frac{3\rho}{\epsilon}\right)^n .
	\end{equation}
	Here, $B_2^n(\rho)$ is defined as $\left\{\mathbf{x} \in \mathbb{R}^n:\|\mathbf{x}\|_2 \leq \rho\right\}$.
\end{lemma}

Next, we introduce four essential lemmas for establishing the
theoretical error bound. Lemma~\ref{lem:1} and Lemma~\ref{lem:conv_compatibility}
provide two necessary inequalities for the proof of Lemma~\ref{lem:covering_PLTD}. Lemma~\ref{lem:covering_PLTD} gives an upper bound for the covering number of
the solution space of the recovery model. Lemma~\ref{lemma08} provides a
framework for proving the error bound of our recovery model.

\begin{lemma}
	\label{lem:1}
	For positive integers $p,n_1,\ldots,n_p$ and a radius $\rho>0$, let
	$
	\Gamma
	=
	\{
	\mathcal{T}
	\in
	\mathbb{R}^{{n_1}
		\times {n_2}
		\times\cdots\times {n_p}}
	:
	\|\mathcal{T}\|_{F}
	\leq
	\rho
	\},
	$
	the \(\epsilon\)-net of \(\Gamma\) is denoted by
	\(\bar{\Gamma}\).
	If
	$
	\bar{\mathcal{T}}\in\bar{\Gamma}
	$
	, we have
	$
	\|\bar{\mathcal{T}}\|_{F}
	\leq
	\rho.
	$
\end{lemma}

\begin{proof}
	Since the \(\epsilon\)-net of a set is a subset of that set, we have
	\[
	\bar{\Gamma}
	\subset
	\Gamma.
	\]
	
	Then we can deduce that
	\[
	\bar{\mathcal{T}}
	\in
	\bar{\Gamma}
	\subset
	\Gamma.
	\]
	
	Therefore, according to the definition of
	\(\Gamma\), we obtain
	\[
	\|\bar{\mathcal{T}}\|_{F}
	\leq
	\rho.
	\]
\end{proof}

\begin{lemma}
	\label{lem:conv_compatibility}
	Let $\mathcal{W}\in
	\mathbb{R}^{k\times k\times c_{\rm in}\times c_{\rm out}}$
	be a convolution kernel and
	$\mathcal{V}\in
	\mathbb{R}^{h\times w\times c_{\rm in}}$
	be an input tensor.
	For the zero-padding convolution operator,
	the following inequality holds,
	\begin{equation}
		\|\mathcal{W}\circledast\mathcal{V}\|_{F}
		\leq
		k
		\|\mathcal{W}\|_{F}
		\|\mathcal{V}\|_{F}.
	\end{equation}
\end{lemma}
\begin{proof}
	Let
	$
	\mathcal{Q}
	=
	\mathcal{W}\circledast\mathcal{V}
	$
	denote the output tensor.
	For each spatial location $i=1,\ldots,h$, $j=1,\ldots,w$ and
	output channel $m=1,\ldots,c_{\rm out}$,
	the convolution operation can be written as
	\begin{equation}\nonumber
		\mathcal{Q}(i,j,m)
		=
		\sum_{a,b,q}
		\mathcal{W}(a,b,q,m)
		\mathcal{V}_{\rm pad}(i+a,j+b,q),
	\end{equation}
	where $a,b$ range over the $k\times k$ kernel support,
	$q=1,\ldots,c_{\rm in}$ is the input-channel index, and
	$\mathcal{V}_{\rm pad}$ denotes the zero-padded version of $\mathcal{V}$.
	Using the Cauchy--Schwarz inequality, we have
	\begin{equation}\nonumber
		\begin{aligned}
			|\mathcal{Q}(i,j,m)|^2
			&\leq
			(
			\sum_{a,b,q}
			\mathcal{W}^2(a,b,q,m)
			)
			(
			\sum_{a,b,q}
			\mathcal{V}_{\rm pad}^2(i+a,j+b,q)
			)
			\\
			&=
			\|
			\mathcal{W}(:,:,,:,m)
			\|_F^2
			\sum_{a,b,q}
			\mathcal{V}_{\rm pad}^2(i+a,j+b,q).
		\end{aligned}
	\end{equation}
	Therefore, summing over all spatial locations and output
	channels gives
	\begin{equation}\nonumber
		\begin{aligned}
			\|\mathcal{W}\circledast\mathcal{V}\|_F^2
			&=
			\sum_{i,j,m}
			|\mathcal{Q}(i,j,m)|^2
			\\
			&\leq
			\sum_{i,j,m}
			\|
			\mathcal{W}(:,:,,:,m)
			\|_F^2
			\sum_{a,b,q}
			\mathcal{V}_{\rm pad}^2(i+a,j+b,q)
			\\
			&=
			\sum_m
			\|
			\mathcal{W}(:,:,,:,m)
			\|_F^2
			\sum_{a,b,q}
			\sum_{i,j}
			\mathcal{V}_{\rm pad}^2(i+a,j+b,q)
			\\
			&\leq
			\sum_m
			\|
			\mathcal{W}(:,:,,:,m)
			\|_F^2
			\sum_{a,b}
			\sum_q
			\sum_{i,j}
			\mathcal{V}^2(i,j,q)
			\\
			&=
			\sum_m
			\|
			\mathcal{W}(:,:,,:,m)
			\|_F^2
			\sum_{a,b}
			\|\mathcal{V}\|_F^2
			\\
			&=
			k^2
			\sum_m
			\|
			\mathcal{W}(:,:,,:,m)
			\|_F^2
			\|\mathcal{V}\|_F^2
			\\
			&=
			k^2
			\|\mathcal{W}\|_F^2
			\|\mathcal{V}\|_F^2 .
		\end{aligned}
	\end{equation}
	Taking the square root on both sides yields
	\begin{equation}\nonumber
		\|\mathcal{W}\circledast\mathcal{V}\|_F
		\leq
		k
		\|\mathcal{W}\|_F
		\|\mathcal{V}\|_F .
	\end{equation}
	This completes the proof.
\end{proof}

\begin{lemma}
	\label{lem:covering_PLTD}
	
	Let
	$
	\mathfrak{L}
	= 
	\{ \mathcal{X}\in\mathbb{R}^{n_1\times n_2\times n_3} :\;
	\mathcal{X}=g_{\boldsymbol{\theta}}
	\big(
	(\mathbf{A}\mathbf{B})\circ\mathbf{c}+ \mathcal{S}
	\big)
	\},
	$
	where $\mathcal{S}\in\mathbb{R}^{n_1\times n_2\times\hat n_3}$ is a fixed pre-trained latent tensor,
	$\mathbf{A}\in\mathbb{R}^{n_1\times r}$,
	$\mathbf{B}\in\mathbb{R}^{r\times n_2}$, and
	$\mathbf{c}\in\mathbb{R}^{\hat n_3}$.
	Here $\circ$ denotes the outer product, 
	$(\mathbf{A}\mathbf{B})\circ\mathbf{c}\in
	\mathbb{R}^{n_1\times n_2\times\hat n_3}$.
	Here, $r$ is the prescribed factorization rank and $\hat n_3$ is the
	latent-channel dimension.
	The learnable transform $g_{\boldsymbol{\theta}}$ is an $L$-layer convolutional neural network with
	$\boldsymbol{\theta}=\{\mathcal{W}_l\}_{l=1}^{L}$:	
	\begin{equation}\nonumber
		\mathcal{Z}_l
		=
		\sigma
		(
		\mathcal{W}_l\circledast\mathcal{Z}_{l-1}
		),
		\quad l=1,\ldots,L-1 ,
	\end{equation}
	where $
	\mathcal{W}_l
	\in
	\mathbb{R}^{k\times k\times c_{l-1}\times c_l}
	$, $k$ is the spatial kernel size, and $c_{l-1}$ and $c_l$
	represent the numbers of input and output channels, respectively, for
	$l=1,\ldots,L$.
	Given the final latent tensor $\mathcal{H}=(\mathbf{A}\mathbf{B})\circ\mathbf{c}+ \mathcal{S}$, we set $\mathcal{Z}_0=\mathcal{H}$, 
	$
	\mathcal{X} = g_{\boldsymbol{\theta}}(\mathcal{H})
	=
	\mathcal{W}_L\circledast\mathcal{Z}_{L-1}.
	$
	
	Assume $\|\mathbf{A}\|_F\leq \alpha_A$, $\|\mathbf{B}\|_F\leq \alpha_B$, $\|\mathbf{c}\|_2\leq \alpha_c$, and $\|\mathcal{W}_l\|_F\leq \beta_l / k$ for $l=1,\ldots,L$. 
	According to Lemma~\ref{lem:conv_compatibility},
	each convolution layer satisfies, for all compatible tensors
	$\mathcal{V}_1$ and $\mathcal{V}_2$,
	\begin{equation}
		\label{eq:conv_lipschitz}
		\|
		\mathcal{W}_l\circledast\mathcal{V}_1
		-
		\mathcal{W}_l\circledast\mathcal{V}_2
		\|_F
		\leq
		\beta_l
		\|\mathcal{V}_1-\mathcal{V}_2\|_F .
	\end{equation}
	
	The activation function is $\eta$-Lipschitz continuous on compatible tensors:
	\begin{equation}
		\label{eq:activation_lipschitz}
		\|\sigma(\mathcal{V}_1)-\sigma(\mathcal{V}_2)\|_F
		\leq
		\eta
		\|\mathcal{V}_1-\mathcal{V}_2\|_F .
	\end{equation}
	
	Moreover, assume the intermediate tensors are bounded:
	\begin{equation}
		\label{eq:feature_bound}
		\|\mathcal{Z}_l\|_F\leq M_l,
		\quad
		l=0,\ldots,L-1 .
	\end{equation}
	Then there exists a constant $\epsilon_0>0$ such that, for any $0<\epsilon\leq\epsilon_0$, the covering number of $\mathfrak{L}$ satisfies
	\begin{equation}
		\mathcal{N}(\mathfrak{L}, \|\cdot\|_F, \epsilon)
		\leq
		\left(
		\frac{C_0}{\epsilon}
		\right)^{
			H_g
			+r(n_1+n_2)+\hat n_3
		},
	\end{equation}
	where
	$
	H_g
	=
	\sum_{l=1}^{L}
	c_lc_{l-1}k^2
	$, $c_0=\hat n_3$, $c_L=n_3$, $C_0>0$ is a constant independent of $\epsilon$,
	which may depend on the parameter bounds and the network architecture.
	
\end{lemma}

\begin{proof}
	
	For arbitrary $0<\epsilon\leq\epsilon_0$, let
	\[
	\begin{aligned}
		\mathfrak{L}_A
		&=
		\{
		\mathbf{A}\in\mathbb{R}^{n_1\times r}:
		\|\mathbf{A}\|_F\leq\alpha_A
		\},
		\\
		\mathfrak{L}_B
		&=
		\{
		\mathbf{B}\in\mathbb{R}^{r\times n_2}:
		\|\mathbf{B}\|_F\leq\alpha_B
		\},
		\\
		\mathfrak{L}_c
		&=
		\{
		\mathbf{c}\in\mathbb{R}^{\hat n_3}:
		\|\mathbf{c}\|_2\leq\alpha_c
		\},
		\\
		\mathfrak{L}_{W_l}
		&=
		\{
		\mathcal{W}_l\in
		\mathbb{R}^{k\times k\times c_{l-1}\times c_l}:
		\|\mathcal{W}_l\|_F\leq\beta_l/k
		\}.
	\end{aligned}
	\]
	
	For arbitrary positive radii $\epsilon_A$, $\epsilon_B$, $\epsilon_c$, and
	$\epsilon_l$, the covering-number bound for Euclidean balls gives corresponding
	$\epsilon_A$-, $\epsilon_B$-, $\epsilon_c$-, and $\epsilon_l$-nets
	$\bar{\mathfrak{L}}_A$, $\bar{\mathfrak{L}}_B$,
	$\bar{\mathfrak{L}}_c$, and $\bar{\mathfrak{L}}_{W_l}$ obey
	\[
	\begin{aligned}
		\mathcal{N}\left(\mathfrak{L}_A, \|\cdot\|_F, \epsilon_A \right)
		&\leq
		\left(
		\frac{3\alpha_A}{\epsilon_A}
		\right)^{n_1r},\\
		\mathcal{N}\left(\mathfrak{L}_B, \|\cdot\|_F, \epsilon_B \right)
		&\leq
		\left(
		\frac{3\alpha_B}{\epsilon_B}
		\right)^{rn_2},
		\\
		\mathcal{N}\left(\mathfrak{L}_c, \|\cdot\|_2, \epsilon_c \right)
		&\leq
		\left(
		\frac{3\alpha_c}{\epsilon_c}
		\right)^{\hat n_3},\\
		\mathcal{N}\left(\mathfrak{L}_{W_l}, \|\cdot\|_F, \epsilon_l \right)
		&\leq
		\left(
		\frac{3\beta_l}{k\epsilon_l}
		\right)^{k^2c_{l-1}c_l},
		\quad l=1,\ldots,L.
	\end{aligned}
	\]
	Now choose the approximation radii
	
	\[\epsilon_A
	=
	\frac{\epsilon}
	{6\alpha_B\alpha_c
		\eta^{L-1}\prod_{l=1}^{L}\beta_l},
	\] 
	\[
	\epsilon_B
	=
	\frac{\epsilon}
	{6\alpha_A\alpha_c
		\eta^{L-1}\prod_{l=1}^{L}\beta_l},
	\]
	\[
	\epsilon_c
	=
	\frac{\epsilon}
	{6\alpha_A\alpha_B
		\eta^{L-1}\prod_{l=1}^{L}\beta_l},
	\]
	\[
	\epsilon_l
	=
	\frac{\epsilon}
	{
		2Lk
		\eta^{L-l}
		\prod_{j=l+1}^{L}\beta_j
		M_{l-1}
	},
	\quad l=1,\ldots,L.
	\]
	
	Here, $\epsilon_0$ is chosen sufficiently small such that
	$\epsilon_A\leq\alpha_A$, $\epsilon_B\leq\alpha_B$,
	$\epsilon_c\leq\alpha_c$, and
	$\epsilon_l\leq\beta_l/k$ for all $l=1,\ldots,L$.
	Accordingly, we construct an $\epsilon_A$-net for $\mathfrak{L}_A$, an
	$\epsilon_B$-net for $\mathfrak{L}_B$, an $\epsilon_c$-net for
	$\mathfrak{L}_c$, and an $\epsilon_l$-net for $\mathfrak{L}_{W_l}$.
	Therefore, there exist
	$\bar{\mathbf{A}}\in\bar{\mathfrak{L}}_A$,
	$\bar{\mathbf{B}}\in\bar{\mathfrak{L}}_B$,
	$\bar{\mathbf{c}}\in\bar{\mathfrak{L}}_c$,
	and
	$\bar{\mathcal{W}}_l\in\bar{\mathfrak{L}}_{W_l}$
	for $l=1,\ldots,L$
	such that
	\[
	\begin{aligned}
		\|\mathbf{A}-\bar{\mathbf{A}}\|_F
		&\leq\epsilon_A,
		&
		\|\mathbf{B}-\bar{\mathbf{B}}\|_F
		&\leq\epsilon_B,
		\\
		\|\mathbf{c}-\bar{\mathbf{c}}\|_2
		&\leq\epsilon_c,
		&
		\|\mathcal{W}_l-\bar{\mathcal{W}}_l\|_F
		&\leq\epsilon_l.
	\end{aligned}
	\]
	
	In the following, we first construct a covering set for
	$
	\{\mathcal{L}\in\mathbb{R}^{n_1\times n_2\times\hat n_3}:
	\mathcal{L}=(\mathbf{A}\mathbf{B})\circ\mathbf{c}\}
	$. Given $\bar{\mathbf{A}}$, $\bar{\mathbf{B}}$, and $\bar{\mathbf{c}}$ as approximations of $\mathbf{A}$, $\mathbf{B}$, and $\mathbf{c}$, respectively, define
	$\bar{\mathcal{L}}=(\bar{\mathbf{A}}\bar{\mathbf{B}})\circ\bar{\mathbf{c}}$.
	The difference between $\mathcal{L}$ and $\bar{\mathcal{L}}$ can be decomposed as
	
	\[
	\begin{aligned}
		\mathcal{L}-\bar{\mathcal{L}}
		=&\,
		(\mathbf{A}\mathbf{B})\circ\mathbf{c}
		-
		(\bar{\mathbf{A}}\bar{\mathbf{B}})
		\circ\bar{\mathbf{c}}
		\\
		=&\,
		\big((\mathbf{A}-\bar{\mathbf{A}})\mathbf{B}\big)
		\circ\mathbf{c}
		+
		\big(\bar{\mathbf{A}}
		(\mathbf{B}-\bar{\mathbf{B}})\big)
		\circ\mathbf{c}
		+
		(\bar{\mathbf{A}}\bar{\mathbf{B}})
		\circ
		(\mathbf{c}-\bar{\mathbf{c}}).
	\end{aligned}
	\]
	
	Then,
	\[
	\begin{aligned}
		\|\mathcal{L}-\bar{\mathcal{L}}\|_F
		\leq\;&
		\|\mathbf{A}-\bar{\mathbf{A}}\|_F
		\|\mathbf{B}\|_F
		\|\mathbf{c}\|_2
		\\
		&+
		\|\bar{\mathbf{A}}\|_F
		\|\mathbf{B}-\bar{\mathbf{B}}\|_F
		\|\mathbf{c}\|_2
		\\
		&+
		\|\bar{\mathbf{A}}\|_F
		\|\bar{\mathbf{B}}\|_F
		\|\mathbf{c}-\bar{\mathbf{c}}\|_2
		\\
		\leq\;&
		\alpha_B\alpha_c\epsilon_A
		+
		\alpha_A\alpha_c\epsilon_B
		+
		\alpha_A\alpha_B\epsilon_c
		\\
		=\;&
		\frac{\epsilon}
		{2\eta^{L-1}\prod_{l=1}^{L}\beta_l}.
	\end{aligned}
	\]
	
	In the theoretical analysis, the pre-trained latent tensor \(\mathcal{S}\) is treated as a fixed tensor.
	Define $\bar{\mathcal{H}}=\bar{\mathcal{L}} + \mathcal{S}$; since
	$\mathcal{H}=\mathcal{L} + \mathcal{S}$, we have
	\begin{equation}
		\|(\mathcal{L}+ \mathcal{S})-(\bar{\mathcal{L}}+ \mathcal{S})\|_F
		\leq
		\frac{\epsilon}{2\eta^{L-1}
			\prod_{l=1}^{L}\beta_l}.
	\end{equation}
	Next, we will deduce that $\mathfrak{L}$ has an $\epsilon$-net 
	$
	\bar{\mathfrak{L}}
	= 
	\big\{ \bar{\mathcal{X}}\in\mathbb{R}^{n_1\times n_2\times n_3} :\;
	\bar{\mathcal{X}}=g_{\bar{\boldsymbol{\theta}}}
	\big(
	(\bar{\mathbf{A}}\bar{\mathbf{B}})
	\circ\bar{\mathbf{c}}+ \mathcal{S}
	\big),\;
	\bar{\mathbf{A}}\in\bar{\mathfrak{L}}_A,\;
	\bar{\mathbf{B}}\in\bar{\mathfrak{L}}_B,\;
	\bar{\mathbf{c}}\in\bar{\mathfrak{L}}_c,\;
	\bar{\boldsymbol{\theta}}=\{\bar{\mathcal{W}}_l\}_{l=1}^{L},\;
	\bar{\mathcal{W}}_l\in\bar{\mathfrak{L}}_{W_l},\;
	l=1,\ldots,L
	\big\}.
	$
	According to the formula of tensor $\mathcal{X}$, we have
	\begin{equation}\nonumber
		\|\mathcal{X}-\bar{\mathcal{X}}\|_F
		\leq
		\|
		g_{\boldsymbol{\theta}}(\mathcal{H})
		-
		g_{\boldsymbol{\theta}}(\bar{\mathcal{H}})
		\|_F
		+
		\|
		g_{\boldsymbol{\theta}}(\bar{\mathcal{H}})
		-
		g_{\bar{\boldsymbol{\theta}}}
		(\bar{\mathcal{H}})
		\|_F .
	\end{equation}
	For fixed network parameters, set
	$\mathcal{Z}_0=\mathcal{H}$ and
	$\widetilde{\mathcal{Z}}_0=\bar{\mathcal{H}}$. For $l=1,\ldots,L-1$, let
	$
	\mathcal{Z}_l
	=
	\sigma(\mathcal{W}_l\circledast\mathcal{Z}_{l-1})
	$, and
	$
	\widetilde{\mathcal{Z}}_l
	=
	\sigma(\mathcal{W}_l\circledast\widetilde{\mathcal{Z}}_{l-1}).
	$
	
	Using
	(\ref{eq:activation_lipschitz})
	and
	(\ref{eq:conv_lipschitz}), we have
	\begin{equation}\nonumber
		\|\mathcal{Z}_l-\widetilde{\mathcal{Z}}_l\|_F
		\leq
		\eta\beta_l
		\|
		\mathcal{Z}_{l-1}
		-
		\widetilde{\mathcal{Z}}_{l-1}
		\|_F .
	\end{equation}
	Therefore,
	\begin{equation}
		\|
		g_{\boldsymbol{\theta}}(\mathcal{H})
		-
		g_{\boldsymbol{\theta}}(\bar{\mathcal{H}})
		\|_F
		\leq
		\eta^{L-1}
		\prod_{l=1}^{L}\beta_l
		\|
		\mathcal{H}-\bar{\mathcal{H}}
		\|_F 
		\leq \frac{\epsilon}{2}.
	\end{equation}
	For the fixed input $\bar{\mathcal{H}}$ and the selected network
	parameter net centers, set
	$
	\mathcal{U}_0=\bar{\mathcal{U}}_0=\bar{\mathcal{H}}
	$
	and, for $l=1,\ldots,L-1$, define
	$\mathcal{U}_l=\sigma(\mathcal{W}_l\circledast\mathcal{U}_{l-1})$ and
	$\bar{\mathcal{U}}_l=\sigma(\bar{\mathcal{W}}_l\circledast\bar{\mathcal{U}}_{l-1})$.
	For $l=0,\ldots,L-1$, set
	$
	E_l
	=
	\|
	\mathcal{U}_l-\bar{\mathcal{U}}_l
	\|_F .
	$
	Then, for $l=1,\ldots,L-1$,
	\begin{equation}\nonumber
		\begin{aligned}
			E_l
			&=
			\left\|
			\sigma
			\left(
			\mathcal{W}_l \circledast \mathcal{U}_{l-1}
			\right)
			-
			\sigma
			\left(
			\bar{\mathcal{W}}_l \circledast \bar{\mathcal{U}}_{l-1}
			\right)
			\right\|_F
			\\
			&\leq
			\eta
			\left\|
			\mathcal{W}_l \circledast \mathcal{U}_{l-1}
			-
			\bar{\mathcal{W}}_l \circledast \bar{\mathcal{U}}_{l-1}
			\right\|_F
			\\
			&=
			\eta
			\left\|
			(\mathcal{W}_l-\bar{\mathcal{W}}_l)
			\circledast
			\mathcal{U}_{l-1}
			+
			\bar{\mathcal{W}}_l
			\circledast
			(
			\mathcal{U}_{l-1}
			-
			\bar{\mathcal{U}}_{l-1}
			)
			\right\|_F
			\\
			&\leq
			\eta
			\|
			(\mathcal{W}_l-\bar{\mathcal{W}}_l)
			\circledast
			\mathcal{U}_{l-1}
			\|_F
			+
			\eta
			\|
			\bar{\mathcal{W}}_l
			\circledast
			(
			\mathcal{U}_{l-1}
			-
			\bar{\mathcal{U}}_{l-1}
			)
			\|_F 
			\\
			&\leq
			\eta k \epsilon_lM_{l-1}
			+
			\eta\beta_lE_{l-1}.
		\end{aligned}
	\end{equation}
	For the last layer, which has no activation function, define
	\begin{equation}\nonumber
		\begin{aligned}
			E_L
			&=
			\|
			\mathcal{W}_L\circledast\mathcal{U}_{L-1}
			-
			\bar{\mathcal{W}}_L\circledast\bar{\mathcal{U}}_{L-1}
			\|_F
			\\
			&\leq
			\|
			(\mathcal{W}_L-\bar{\mathcal{W}}_L)
			\circledast
			\mathcal{U}_{L-1}
			\|_F
			+
			\|
			\bar{\mathcal{W}}_L
			\circledast
			(
			\mathcal{U}_{L-1}
			-
			\bar{\mathcal{U}}_{L-1}
			)
			\|_F
			\\
			&\leq
			k\epsilon_LM_{L-1}
			+
			\beta_LE_{L-1}.
		\end{aligned}
	\end{equation}
	Expanding the recursion gives
	\begin{equation}\nonumber
		E_L
		\leq
		\sum_{l=1}^{L}
		k\eta^{L-l}
		\prod_{j=l+1}^{L}\beta_j
		\epsilon_lM_{l-1}.
	\end{equation}
	With the above choice of $\epsilon_l$,
	$
	E_L
	\leq
	\sum_{l=1}^{L}
	\frac{\epsilon}{2L}
	=
	\frac{\epsilon}{2}.
	$
	Thus,
	\begin{equation}
		\|
		g_{\boldsymbol{\theta}}(\bar{\mathcal{H}})
		-
		g_{\bar{\boldsymbol{\theta}}}
		(\bar{\mathcal{H}})
		\|_F
		\leq
		\frac{\epsilon}{2}.
	\end{equation}
	Combining two parts,
	$
	\|\mathcal{X}-\bar{\mathcal{X}}\|_F
	\leq
	\epsilon .
	$
	According to the above formulas, $\bar{\mathfrak{L}}$ is an $\epsilon$-net of
	$\mathfrak{L}$. Hence, the covering
	number
	\begin{equation}
		\begin{aligned}
			\mathcal{N}(\mathfrak{L},\|\cdot\|_F,\epsilon)
			&\leq
			\mathcal{N}\left(\mathfrak{L}_A, \|\cdot\|_F, \epsilon_A \right)
			\mathcal{N}\left(\mathfrak{L}_B, \|\cdot\|_F, \epsilon_B \right)
			\\
			&\quad\times
			\mathcal{N}\left(\mathfrak{L}_c, \|\cdot\|_2, \epsilon_c \right)
			\\
			&\quad\times
			\prod_{l=1}^{L}
			\mathcal{N}\left(\mathfrak{L}_{W_l}, \|\cdot\|_F, \epsilon_l \right)
			\\
			&\leq
			\left(
			\frac{3\alpha_A}{\epsilon_A}
			\right)^{n_1r}
			\left(
			\frac{3\alpha_B}{\epsilon_B}
			\right)^{rn_2}
			\left(
			\frac{3\alpha_c}{\epsilon_c}
			\right)^{\hat n_3}
			\\
			&\quad\times
			\prod_{l=1}^{L}
			\left(
			\frac{3 \beta_l}{k\epsilon_l}
			\right)^{
				k^2c_{l-1}c_l
			}
			\\
			&\leq
			\left(
			\frac{C_0}{\epsilon}
			\right)^{
				H_g
				+r(n_1+n_2)+\hat n_3
			}.
		\end{aligned}
	\end{equation}
	Therefore, Lemma~\ref{lem:covering_PLTD} holds.
	
\end{proof}

\begin{lemma}\cite{fan2021multi}\label{lemma08}
	Let $\mathfrak{L}$ be a set of tensors of size
	$n_1\times n_2\times n_3$, let
	$\mathcal{Y}\in\mathbb{R}^{n_1\times n_2\times n_3}$, and set
	$N=n_1n_2n_3$.
	For $\epsilon>0$, denote the $\epsilon$-covering number of $\mathfrak{L}$
	under the Frobenius norm by
	$\mathcal{N}(\mathfrak{L},\|\cdot\|_F,\epsilon)$.
	Let $\Omega\subseteq
	\{1,\ldots,n_1\}\times\{1,\ldots,n_2\}\times\{1,\ldots,n_3\}$
	be a randomly sampled observation index set and let $|\Omega|$ be its cardinality.
	Suppose that
	$
	\|\mathcal{Y}\|_{\infty}\leq\delta
	$, $
	\sup_{\mathcal{X}\in \mathfrak{L}}
	\|\mathcal{X}\|_{\infty}\leq\delta
	$, then
	\begin{equation*}
		\begin{aligned}
			&\sup_{\mathcal{X}\in\mathfrak{L}}
			\left|
			\frac{\|\mathcal{Y}-\mathcal{X}\|_F}{\sqrt{N}}
			-
			\frac{\|\mathcal{P}_{\Omega}(\mathcal{Y}-\mathcal{X})\|_F}{\sqrt{|\Omega|}}
			\right|
			\\
			&\quad\leq
			\frac{2\epsilon}{\sqrt{|\Omega|}}
			+
			\left(
			\frac{8\delta^4}{|\Omega|}
			\left[
			\log\mathcal{N}(\mathfrak{L},\|\cdot\|_F,\epsilon)+\log N
			\right]
			\right)^{1/4}.
		\end{aligned}
	\end{equation*}
	with probability at least $1-2 N^{-1}$, where
	$\mathcal{P}_{\Omega}(\cdot)$ retains the entries indexed by $\Omega$ and sets the others to zero. See \cite{fan2021multi} for the detailed proof of Lemma~\ref{lemma08}.
\end{lemma}

\begin{theorem}[Recovery Error Bound] \label{thm:baseline1}
	Let $\hat{\mathcal{X}}\in\mathbb{R}^{n_1\times n_2\times n_3}$ be the underlying tensor. For any fixed pre-trained latent tensor
	$\mathcal{S}\in
	\mathbb{R}^{n_1\times n_2\times\hat n_3}$,
	let $\mathcal{X}$ be recovered from entries indexed by the random observation
	index set
	$\Omega\subseteq
	\{1,\ldots,n_1\}\times\{1,\ldots,n_2\}\times\{1,\ldots,n_3\}$.
	Consider the model class
	\[
	\mathfrak{L}_{pre}
	:=
	\left\{
	\mathcal{X}\in\mathbb{R}^{n_1\times n_2\times n_3}
	:\;
	\mathcal{X}=g_{\boldsymbol{\theta}}\!\big((\mathbf{A}\mathbf{B})\circ\mathbf{c}+ \mathcal{S}\big)
	\right\},
	\]
	where $\mathbf{A}\in\mathbb{R}^{n_1\times r}$,
	$\mathbf{B}\in\mathbb{R}^{r\times n_2}$,
	$\mathbf{c}\in\mathbb{R}^{\hat n_3}$, $\circ$ denotes the outer product,
	and $g_{\boldsymbol{\theta}}$ is the learnable transform with
	$\boldsymbol{\theta}=\{\mathcal{W}_l\}_{l=1}^{L}$. Assume that
	all the hypotheses of Lemma~\ref{lem:covering_PLTD} hold for
	$\mathfrak{L}_{pre}$; namely,
	$\|\mathbf{A}\|_F\leq\alpha_A$,
	$\|\mathbf{B}\|_F\leq\alpha_B$,
	$\|\mathbf{c}\|_2\leq\alpha_c$,
	$g_{\boldsymbol{\theta}}$ has the $L$-layer convolutional architecture specified in
	Lemma~\ref{lem:covering_PLTD},
	$\|\mathcal{W}_l\|_F\leq\beta_l / k$ for $l=1,\ldots,L$,
	the activation function is $\eta$-Lipschitz continuous, and the intermediate
	tensors satisfy $\|\mathcal{Z}_l\|_F\leq M_l$ uniformly over all admissible
	inputs and network parameters for $l=0,\ldots,L-1$. Moreover, suppose that $\Omega$ is randomly sampled. In addition, assume that
	\[
	\|\hat{\mathcal{X}}\|_{\infty}\leq\delta,
	\qquad
	\sup_{\mathcal{X}\in \mathfrak{L}_{pre}}
	\|\mathcal{X}\|_{\infty}\leq\delta.
	\]
	Set $N=n_1n_2n_3$,
	$H_g=\sum_{l=1}^{L}c_lc_{l-1}k^2$, and
	$\mathcal{C}=r(n_1+n_2)+\hat n_3+H_g$, and let $C_0>0$ be the
	constant from Lemma~\ref{lem:covering_PLTD}.
	Then there exists a constant $\epsilon_0>0$ such that, for any $0<\epsilon\leq\epsilon_0$, with probability at least $1-2N^{-1}$,
	\begin{equation}
		\begin{aligned}
			\frac{\|\hat{\mathcal{X}}-\mathcal{X}\|_F}{\sqrt{N}}
			&\le \frac{\|\mathcal{P}_{\Omega}(\hat{\mathcal{X}}-\mathcal{X})\|_F}{\sqrt{|\Omega|}} + \frac{2\epsilon}{\sqrt{|\Omega|}} \\
			&\quad + \left( \frac{8\delta^4 \left(\mathcal{C} \log\left(\frac{C_0}{\epsilon}\right) + \log N \right)}{|\Omega|} \right)^{1/4}.
		\end{aligned}
	\end{equation}
\end{theorem}

\begin{proof}
	By the covering number bound established in the previous lemma~\ref{lem:covering_PLTD}, we have
	\[
	\mathcal{N}(\mathfrak{L}_{pre},\|\cdot\|_F,\epsilon)
	\le
	\left(
	\frac{C_0}{\epsilon}
	\right)^{
		\mathcal{C}
	}.
	\]
	Hence,
	\[
	\begin{aligned}
		\log \mathcal{N}(\mathfrak{L}_{pre},\|\cdot\|_F,\epsilon)
		&\le
		\mathcal{C}
		\log\left(\frac{C_0}{\epsilon}\right),
	\end{aligned}
	\]
	
	Applying Lemma~\ref{lemma08} with
	$\mathcal{Y}=\hat{\mathcal{X}}$ and $\mathfrak{L}=\mathfrak{L}_{pre}$, we obtain
	\begin{equation}
		\begin{aligned}
			\frac{\|\hat{\mathcal{X}}-\mathcal{X}\|_F}{\sqrt{N}}
			&\le \frac{\|\mathcal{P}_{\Omega}(\hat{\mathcal{X}}-\mathcal{X})\|_F}{\sqrt{|\Omega|}} + \frac{2\epsilon}{\sqrt{|\Omega|}} \\
			&+ \left( \frac{8\delta^4 \left( \mathcal{C} \log\left(\frac{C_0}{\epsilon}\right) + \log N \right)}{|\Omega|} \right)^{1/4}.
		\end{aligned}
	\end{equation}
	This completes the proof.
\end{proof}

\begin{remark}[Advantage of Pre-Trained Common Structure in Tensor Modeling]
	\label{rem:pretrained_advantage}
	The pre-trained latent tensor $\mathcal{S}$ provides the common structure
	inherited from the pre-trained large vision model. Since $\mathcal{S}$ is fixed,
	it introduces no additional parameter dimension in the covering-number exponent.
	Therefore, PLTD only needs to characterize the instance-specific structure of
	the target tensor through the learnable low-rank latent tensor
	\[
	\mathcal{L}
	=
	(\mathbf{A}\mathbf{B})\circ\mathbf{c},
	\]
	which is parameterized by
	$r(n_1+n_2)+\hat n_3$ scalar parameters. According to
	Lemma~\ref{lem:covering_PLTD}, the resulting model class satisfies
	\[
	\log\mathcal{N}
	(\mathfrak{L}_{pre},\|\cdot\|_F,\epsilon)
	\leq
	\big(
	H_g+r(n_1+n_2)+\hat n_3
	\big)
	\log\left(\frac{C_0}{\epsilon}\right).
	\]
	
	In contrast, most untrained deep tensor decomposition methods rely on complex neural networks to model each
	target tensor from scratch. Let $d_{\mathrm{un}}$ denote the number of 
	learnable parameters required for latent tensor modeling. In general,
	\[
	d_{\mathrm{un}}
	\gg
	r(n_1+n_2)+\hat n_3.
	\]

	This indicates that, by leveraging the
	fixed pre-trained latent tensor $\mathcal{S}$ to provide informative common
	structures, PLTD avoids modeling each target tensor entirely from scratch and only requires a lightweight low-rank latent
	tensor to characterize the instance-specific structure. This
	substantially reduces the model complexity while retaining strong tensor
	modeling capability. Consequently, PLTD achieves an unprecedented balance among higher recovery fidelity, fewer learnable parameters, and smaller carbon footprint.
\end{remark}
\end{document}